\documentclass[12pt]{amsart}

\usepackage{Style_files/Packages}
\usepackage{Style_files/New_commands}
\usepackage{Style_files/Style}
\usepackage{comment}
\usepackage{thmtools}
\usepackage{thm-restate}

\usepackage{pgf, tikz}
\usetikzlibrary{arrows, automata, positioning}

\newcommand{\LPG}{\mathsf{L}\mathsf{P}\mathsf{G}}

\renewcommand{\S}{\mathcal{S}_{d-1}}
\renewcommand{\P}{\mathbb{P}}

\usepackage{accents}

\title{On Single Index Models beyond Gaussian Data}
\author{Joan Bruna$^{\ast, \star, \dagger}$, Loucas Pillaud-Vivien$^{\ast, \dagger}$ and Aaron Zweig$^{\ast}$}
\address{$^\ast$ Courant Institute of Mathematical Sciences,  New York University}
\address{$^\star$ Center for Data Science, New York University}
\address{$^\dagger$ Center for Computational Mathematics, Flatiron Institute}

\begin{document}

\maketitle

\begin{abstract}
Sparse high-dimensional functions have arisen as a rich framework to study the behavior of gradient-descent methods using shallow neural networks, showcasing their ability to perform feature learning beyond linear models. 
Amongst those functions, the simplest are single-index models $f(x) = \phi( x \cdot \theta^*)$, where the labels are generated by an arbitrary non-linear scalar link function $\phi$ applied to an unknown one-dimensional projection $\theta^*$ of the input data. By focusing on Gaussian data, several recent works have built a remarkable picture, where the so-called information exponent (related to the regularity of the link function) controls the required sample complexity. In essence, these tools exploit the stability and spherical symmetry of Gaussian distributions. In this work, building from the framework of \cite{arous2020online}, we explore extensions of this picture beyond the Gaussian setting, where both stability or symmetry might be violated. Focusing on the planted setting where $\phi$ is known, our main results establish that Stochastic Gradient Descent can efficiently recover the unknown direction $\theta^*$ in the high-dimensional regime, under assumptions that extend previous works ~\cite{yehudai2020learning,wu2022learning}. 
\end{abstract}


\setcounter{tocdepth}{1}
\tableofcontents



\section{Introduction}
Over the past years, there has been sustained effort to enlarge our mathematical understanding of high-dimensional learning, particularly when using neural networks trained with gradient-descent methods --- highlighting the interplay between algorithmic, statistical and approximation questions. An essential, distinctive aspect of such models is their ability to perform \emph{feature learning}, or to extract useful low-dimensional features out of high-dimensional observations. 

An appealing framework to rigorously analyze this behavior are sparse functions of the form $f(x) = \phi( \Theta_*^\top x)$, where the labels are generated by a generic non-linear, low-dimensional function $\phi: \mathbb{R}^k \to \mathbb{R}$ of linear features $\Theta^\top_* x$, with $\Theta_* \in \mathbb{R}^{d \times k}$ with $k \ll d$. While the statistical and approximation aspects of such function classes are by now well-understood \cite{barron1993universal,bach2017breaking}, the outstanding challenge remains computational, in particular in understanding the ability of gradient-descent methods to succeed. Even in the simplest setting of \emph{single-index models} ($k=1$), and assuming that $\phi$ is known,  the success of gradient-based learning depends on an intricate interaction between the data distribution $x \sim \nu$ and the `link' function $\phi$; and in fact computational lower bounds are known for certain such choices \cite{yehudai2020learning,song2021cryptographic,goel2020superpolynomial,diakonikolas2017statistical,shamir2018distribution}. 

Positive results thus require to make specific assumptions, either about the data, or about the link function, or both. On one end, there is a long literature, starting at least with \cite{kalai2009isotron,shalev2010learning,kakade2011efficient}, that exploits certain properties of $\phi$, such as invertibility or monotonicity, under generic data distributions satisfying mild anti-concentration properties \cite{soltanolkotabi2017learning,frei2020agnostic,yehudai2020learning,wu2022learning}. On the other end, by focusing on canonical high-dimensional measures such as the Gaussian distribution, the seminal works \cite{arous2020online,dudeja2018learning} built a harmonic analysis framework of SGD, resulting in a fairly complete picture of the sample complexity required to learn generic link functions $\phi$, and revealing a rich asymptotic landscape beyond the proportional regime $n \asymp d$, characterized by the number of vanishing moments, or \emph{information exponent} $s$ of $\phi$, 
whereby $n \asymp d^{s-1}$ samples are needed for recovery. Since then, several authors have built and enriched this setting to multi-index models~\cite{abbe2022merged,abbe2023sgd,damian2022neural,arnaboldi2023high}, addressing the semi-parametric learning of the link function~\cite{biettilearning2022}, as well as exploring SGD-variants~\cite{arous2022high,barak2022hidden,berthier2023learning,chen2023learning}. This harmonic analysis framework relies on two key properties of the Gaussian measure and their interplay with SGD: its spherical symmetry and its stability by linear projection. Together, they provide an optimization landscape that is well-behaved in the limit of infinite data, and enable SGD to escape the `mediocrity' of initialisation, where the initial direction $\theta_0$, in the high-dimensional setting, has vanishingly small correlation $| \theta_0 \cdot \theta^*| \simeq 1/\sqrt{d}$ with the planted direction $\theta^*$. 

In this work, we study to what extent the `Gaussian picture' is robust to perturbations, focusing on the planted setting where $\phi$ is known. Our motivation comes from the fact that real data is rarely Gaussian, yet amenable to being approximately Gaussian via CLT-type arguments.
We establish novel positive results along two main directions: (i) when spherical symmetry is preserved but stability is lost, and (ii) when spherical symmetry is lost altogether. In the former, we show that spherical harmonics can be leveraged to provide a benign optimization landscape for SGD under mild regularity assumptions, for initialisations that can be reached with constant probability with the same sample complexity as in the Gaussian case. 
In the latter, we quantify the lack of symmetry with robust projected Wasserstein distances, and show that for `quasi-symmetric' measures with small distance to the Gaussian reference, SGD efficiently succeeds for link functions with information exponent $s \leq 2$. Finally, using Stein's method, we address substantially `non-symmetric' distributions, demonstrating the strength and versatility of the harmonic analysis framework.

\section{Preliminaries and Problem Setup}
\label{sec:preliminaries}

The focus of this work is to understand regression problems with input/output data $(x, y) \in \R^d \times \R $ generated by \textit{single-index} models. This is a class of problems where the data labels are produced by a non-linear map of a one-dimensional projection of the input, that is
\begin{align}
\label{eq:y_model}
    y = \phi(x \cdot \theta^*), 
\end{align}
where $\phi:\mathbb{R} \to \mathbb{R}$ is also known as the \textit{link function}, and $\theta^* \in \S$, the sphere of $\R^d$, is the \textit{hidden direction} that the models wants to learn. Quite naturally, the learning is made through the family of generalized linear predictors $\mathcal{H} = \{\phi_{\theta} : x \to \phi( x \cdot \theta), \ \text{for }\, \theta\! \in \S \}$, built upon the link function (which is assumed known) and parametrized by the sphere.
\paragraph{Loss function.}
We assume that the input data is distributed according to a probability $\nu \in \mathcal{P}(\R^d)$. Equation~\eqref{eq:y_model} implies that the target function that produces the labels, $\phi_{\theta^*},$ lies in this parametric class. The overall loss classically corresponds to the average over all the data of the square penalisation $l(\theta, x) := (\phi_\theta(x) - \phi_{\theta^*}(x))^2$ so that the population loss writes 
\begin{align}
\label{eq:mainsetup}
    L(\theta) := \mathbb{E}_{\nu} \left[\big(\phi(x \cdot \theta) - \phi(x \cdot \theta^*)\big)^2\right] = \left\|\phi_{\theta} - \phi_{\theta^*}\right\|^2_{L^2_\nu},
\end{align}
where we used the notation $\|f\|^p_{L^p_\nu} = \E_\nu[|f|^p]$, valid for all $p \in \N^*$. Let us put emphasis on the fact that the loss $L$ is a non-convex function of the parameter $\theta$, hence it is not \textit{a priori} guaranteed that gradient-based method are able to retrieve the ground-truth $\theta^*$. This often requires a precise analysis of the \textit{loss landscape}, and where the high-dimensionality can play a role of paramount importance: we place ourselves in this high-dimensional setting for which the dimension is fixed but considered very large $d \gg 1$. Finally, we assume throughout the article that $\phi_\theta$ belongs to the weighted Sobolev space $W^{1,4}_\nu:=\{ \phi, \text{ such that } 
\sup_{\theta \in \S} \left[\|\phi_\theta\|_{L^4_\nu} + \|\phi_\theta'\|_{L^4_\nu}\right] <  \infty \}$.

\paragraph{Stochastic gradient descent.}
To recover the signal given by $\theta^* \in \S$, we run \textit{online stochastic gradient descent} (SGD) on the sphere $\S$. This corresponds to having at each iteration $t \in \mathbb{N}^*$ a \textit{fresh sample} $x_t$ drawn from $\nu$ and independent of the filtration $\mathcal{F}_t = \sigma(x_1, \dots, x_{t-1})$ and performing a spherical gradient step, with step-size $\delta>0$, with respect to $\theta \to l(\theta, x_t)$:
\begin{align}
    \label{eq:SGD}
    \hspace{3cm} \theta_{t+1} &= \frac{\theta_t - \delta \nabla_\theta^\mathcal{S}  l(\theta_t, x_t)}{\left|\theta_t - \delta \nabla_\theta^\mathcal{S}  l(\theta_t, x_t)\right|}, \qquad \text{ with initialization } \theta_{0} \sim \mathrm{Unif}(\S),
\end{align}
%
An important scalar function that enables to track the progress of the SGD iterates is the correlation with the signal $m_\theta:=\theta \cdot \theta^* \in [-1,1]$. We will drop the subscript in case there is no ambiguity. Note that, due to the high-dimensionality of the setup, we have the following lemma:
%
\begin{lemma}
\label{lem:initialization}
For all $a > 0$, we have $ \P_{\theta_0}(m_{\theta_0} \geq a/\sqrt{d}) \leq a^{-1}e^{-a^2/4}  $. Additionally, for any~$\delta>0$ such that $ \max\{a, \delta\} \leq \sqrt{d}/4$, we have the lower bound:
    $\P_{\theta_0}(m_{\theta_0} \geq a/\sqrt{d}) \geq \frac{\delta}{4} e^ {- (a + \delta)^2} $.
\end{lemma}

%
This fact implies that, when running the algorithm in practice, it is initialized with high probability near the equator of $\S$, or at least in a band of typical size $1/\sqrt{d}$ (see Figure~\ref{fig:nice_tikz} for a schematic illustration of this fact). Finally, we use the notation $\nabla_\theta^\mathcal{S}$ to denote the spherical gradient, that is $ \nabla_\theta^\mathcal{S} l(\theta, x) = \nabla_\theta  l(\theta, x) - (\nabla_\theta  l(\theta, x) \cdot \theta) \theta$. As $\nabla_\theta^\mathcal{S}  l(\cdot, x_t)$ is an unbiased estimate of $\nabla_\theta^\mathcal{S}  L$, it is expected that the latter gradient field rules how the SGD iterates travel across the loss landscape.

\paragraph{Loss landscape in the Gaussian case.} As stressed in the introduction, this set-up has been studied by \cite{dudeja2018learning,arous2020online} in the case where $\nu$ is the standard Gaussian, noted as $\gamma$ here to avoid any confusion for later. Let us comment a bit this case to understand what can be the typical landscape of this single-index problem. 
Thanks to the spherical symmetry, the loss admits a scalar summary statistic, given precisely by the correlation $m_\theta$. Moreover, the loss admits an explicit
representation in terms of the Hermite decomposition of the link function $\phi$: if $\{h_j\}_j$ denotes the orthonormal basis of Hermite polynomials of $L^2_\gamma$, then $L(\theta) = 2\sum_j |\langle \phi, h_j\rangle|^2(1 - m^j) := \bar{\ell}(m)$. As a result, 
the gradient field projected along the signal is a (locally simple) \textit{positive function} of the correlation that behaves similarly to
\begin{align}
    \label{eq:loss_property_gaussian}
    - \nabla_\theta^\mathcal{S}  L(\theta) \cdot \theta^* \simeq C m^{s-1} (1 - m),
\end{align}
where $s \in \N^*$ is the index of the first non-zero of the Hermite coefficients $\{\langle \phi, h_j\rangle\}_j$.
This has at least three important consequences for the gradient flow: (i) if initialized positively, the correlation is an increasing function along the dynamics and there is no bad local minima in the loss landscape,  (ii) the parameter $s \in \N^*$ controls the flatness of the loss landscape near the origin and therefore controls the optimization speed of SGD in this region (iii) as soon as $m$ is large enough, the contractive term $1-m$ makes the dynamics converge exponentially fast.
\paragraph{Loss landscape in general cases.} Obviously for general distributions $\nu$, the calculation presented in Eq.\eqref{eq:loss_property_gaussian} is no-longer valid. However, the crux of the present paper is that properties (i)-(ii)-(iii) are robust to the change of distribution and can be shown to be preserved under small adaptations. More precisely, we have the following definition.
\begin{definition}[Local Polynomial Growth]
\label{def:growth}
We say that $L$ has the local polynomial growth of order $k \in \N^*$ and scale $b \geq 0$, if there exists $C > 0$ such that for all $m_\theta \geq b$, 
\begin{equation}
    - \nabla_\theta^\mathcal{S}  L(\theta) \cdot \theta^* \geq C (1 - m_\theta) \left(m_\theta - b\right)^{k-1}~.
\end{equation}
    In such a case we say that $L$ satisfies $\LPG(k,b)$.
\end{definition}
In this definition, and as showed later in specific examples given in Section~\ref{sec:population_landscape}, the scale parameter $b$ should be thought as a small parameter  proportional to $1/\sqrt{d}$. 
%
%
If $\nu$ is Gaussian, we can rewrite Eq.\eqref{eq:loss_property_gaussian} and show that $L$ verifies $\LPG(s,0)$ for $s \in \N^*$, referred to as the \textit{information exponent} of the problem in~\cite{arous2020online}.
An important consequence of satisfying $\LPG(k,b)$ is that the the population landscape is free of bad local minima outside the equatorial band $\Sigma_b:= \{ \theta \in \S \ , \   m_\theta \leq b\}$. Therefore, when $b$ is of scale $1/\sqrt{d}$, Lemma \ref{lem:initialization} indicates that one can efficiently produce initializations that avoid it.  
Hence, this property is the fundamental ingredient that enables the description of the path taken by SGD that we derive it in the next Section. Section~\ref{sec:population_landscape} is devoted to showcasing generic examples when this property is satisfied.
%
%

\section{Stochastic Gradient Descent under \texorpdfstring{$\LPG$}{LPG}}
\label{sec:sgd}

In this section, we derive the main results on the trajectory of the stochastic gradient descent. They state that the property $\LPG(s,b/\sqrt{d})$ is in fact \textit{sufficient} to recover the same quantitative guarantees as the one depicted in \cite{arous2020online}, despite the lack of Gaussianity of the distribution $\nu$. Recall that the recursion satisfied by the SGD iterates is given by Eq.\eqref{eq:SGD}. To describe their movement, let us introduce the following notations: for all $t \in \mathbb{N}^*$, we denote the normalization by $r_t := \left|\theta_t - \delta \nabla_\theta^\mathcal{S}  \ell(\theta_t, x_t)\right|$ and the martingale induced by the stochastic gradient descent as $M_t := l(\theta_t, x_t) - \E_\nu [ l(\theta_t, x)]$.

\paragraph{Moment growth assumptions.} To be able to analyse the SGD dynamics, we make the following assumptions on the moments of the martingale increments induced by the random sampling. To shorten notations, let us denote for all $\theta \in \R^d$, $x_\theta = x \cdot \theta \in \R$ and $C(u,v) = \phi'(u)  \phi(v)$, for $u,v \in \R$. 
\begin{assumption}[Moment Growth Assumption]
\label{ass:momentgrowth}
There exists a constant $K > 0$, independent of the dimension $d$, such that
\begin{align}
    \sup_{\theta \in \mathcal{S}_{d-1}}  \E_x \left[x_{\theta_*}^2 C^2(x_\theta, x_{\theta_*})  \right] \vee  \E_x \left[x_{\theta}^2\, C^2(x_\theta, x_{\theta_*})  \right] &\leq K, \hspace{2.125cm} \text{ and }, \\
    \sup_{\theta \in \mathcal{S}_{d-1}} \E_x \left[|x|^{2k} C^2(x_\theta, x_{\theta_*}) \right] &\leq Kd^k, \qquad \text{ for } k = 1,\,2.
\end{align}
\end{assumption}
%
A precise care is given to the dependency in the dimension in the upper bound to match the practical cases that we later discuss in Section~\ref{sec:population_landscape}. Note that these assumptions are {\color{black} typically true} for sub-gaussian random variables if $\phi$ belongs to a Sobolev regularity class. These assumptions are similar to the one given in Eqs.~(1.3)-(1.4) in \cite{arous2021online} for the Gaussian case. In all the remainder of the section we assume that Assumption \ref{ass:momentgrowth} is satisfied. 

\paragraph{Tracking the correlation.} Recall that the relevant signature of the dynamics is the one-dimensional correlation $m_{t} = \theta_t \cdot \theta^*$. For infinitesimal step-sizes $\delta \to 0$, it is expected that $\theta_t$ follow the spherical gradient flow $\dot{\theta}_t = - \nabla_\theta^\mathcal{S} L(\theta_t)$, that translates naturally on the summary statistics $m_t$ as the following time evolution
\begin{align}
\label{eq:idealized_ODE}
\dot{m_t} = - \nabla_\theta^\mathcal{S} L(\theta_t) \cdot \theta^*.
\end{align}
The main idea behind the result of this section is to show that, even if the energy landscape near $m = m_0$ is rough at scale $1/\sqrt{d}$, the noise induced by SGD does not prevent $m$ to \textit{grow as the idealized dynamics described by the ODE \eqref{eq:idealized_ODE}}. Let us write the iterative recursion followed by $(m_t)_{t \geq 0}$: with the notation recalled above, for $t \in \mathbb{N}^*$, we have 
\begin{align}
\label{eq:dynamics_main_m}
m_{t+1} =  \frac{1}{r_t}\left( m_t - \delta\nabla_\theta^S L(\theta_t) \cdot \theta^* - \delta\nabla_\theta^S M_t \cdot \theta^*  \right).
\end{align}
With this dynamics at hand, the proof consists in controlling both the discretization error through $r_t$ and the directional martingale induced by the term $\delta\nabla_\theta^S M_t \cdot \theta^*$.

\paragraph*{Weak recovery.} As it is the case for the gradient flow, most of the time spent by the SGD dynamics is near the equator, or more precisely in a band of the type $\Sigma_{b,c} = \{ \theta \in \S, \ b/\sqrt{d} \leq m_\theta \leq c/\sqrt{d} \}$, where $b<c$ are constants independent of the dimension. Hence, the real first step of the dynamics is to go out any of these bands. This is the reason why it is natural to define $S_a := \{\theta \in \S, \, m_\theta \geq a\}$, the spherical cap of level $a \in (0,1)$ as well as the hitting time
\begin{align}
    \tau^+_a := \inf\{ t \geq 0, \, m_{\theta_t} \geq a  \},
\end{align}
which corresponds to the first time $(\theta_t)_{t \geq 0}$ enters in $S_a$. We arbitrarily choose a numerical constant independent of dimension, say $a = 1/2$, and refer to the related hitting time $\tau^+_{1/2}$ as the \textit{weak recovery time} of the algorithm.
\begin{theorem}[Weak Recovery]
\label{thm:weak_recovery}
    Let $(\theta_t)_{t \geq 0}$ follow the SGD dynamics of Eq.\eqref{eq:SGD} and let $L$ satisfy $\LPG(s,b/\sqrt{d})$, with~$b >0 $ and $s \in \N^*$, then, conditionally on the fact that $m_0 \geq 5b/\sqrt{d} $, for any $ 0<\varepsilon \leq \varepsilon_*$, we have
    \begin{align}
       \quad \tau^+_{1/2} \leq \left\{
       \begin{array}{llll}
         d \cdot \mathsf{K}/\varepsilon & \text{ when $s = 1,\ \ \ $ and with the choice} & \delta = \varepsilon/d \\
         d \log^2(d) \cdot \mathsf{K}/\varepsilon & \text{ when $s = 2,\ \ \ $ and with the choice} & \delta = \varepsilon/(d \log d) \\
         d^{s-1} \cdot \mathsf{K}/\varepsilon & \text{ when $s \geq 3,\ \ \ $ and with the choice} & \delta = \varepsilon d^{-s/2}
    \end{array} 
    \right.
    \end{align}
    with probability at least $1 - \mathsf{K}\varepsilon$, for generic constants $\mathsf{K}, \varepsilon_* >0$ that depend solely on the link function $\phi$ and the distribution $\nu$.
\end{theorem}
Let us comment on this result. It says that that the integer $s$ coming from the growth condition controls the hardness of exiting the equator of the sphere. Indeed, as can be seen in $\LPG(s,b/\sqrt{d})$, the larger the $s$, the smaller the gradient projection is and hence the less information the SGD dynamics has to move from the initialization. This result can be seen as an extension of \cite[Theorem 1.3]{arous2020online} valid only in the Gaussian case ($b = 0$). Furthermore, the Gaussian case shows that Theorem~\ref{thm:weak_recovery} is tight up to $\log(d)$ factors. Finally, note that the result is conditional to the fact that the initialization is larger that some constant factor of $1/\sqrt{d}$, which has at least constant probability to happen in virtue of Lemma~\ref{lem:initialization}. This probability can be lowered by any constant factor by sampling offline a constant factor (independent of $d$) of i.i.d. initializations and keeping the one that maximizes its correlation with $\theta^*$.
Finally, note that in all the cases covered by the analysis, it is possible to keep track of the constant $C$, and show that overall it depends only (i) on the property of $\nu$ w.r.t. the Gaussian on the one hand and (ii) on the Sobolev norm of the link function $\|\phi\|_{W_\nu^{1,4}}$ on the other hand. 

\paragraph*{Strong recovery.} We place ourselves \textit{after} the weak recovery time described in the previous question and want to understand if the $(\theta_t)_{t \geq 0}$ dynamics goes to $\theta^*$ and if yes, how fast it does so. This is what we refer to as \textit{the strong recovery} question, captured by the fact that the one-dimensional summary statistics  $m$ go towards $1$. Thanks to the Markovian property of the SGD dynamics, we have the equality between all time $s > 0$ marginal laws of
\begin{align*}
\left(\theta_{\tau^+_{\nicefrac{1}{2}} + s}\  \bigg| \  \tau^+_{1/2},\, \theta_{\tau^+_{1/2}}\right) \overset{\text{Law}}{=} \left(\theta_{s} \  \bigg| \  \theta_{s} = \theta_{\tau^+_{1/2}} \right),  
\end{align*}
and hence the strong recovery question is equivalent to study the dynamics with initialization that has already weakly recovered the signal, i.e. such that $m_\theta  = 1/2$. We show that this part of the loss landscape is very different that the equator band in which the dynamics spends most of its times: in all the cases, we can choose stepsizes independent of the dimension and show that the time to reach the vicinity of $\theta^*$ will be independent of $d$.
\begin{theorem}[Strong Recovery]
\label{thm:strong_recovery}
    Let $(\theta_t)_{t \geq 0}$ follow the SGD dynamics of Eq.\eqref{eq:SGD} and let $L$ satisfy $\LPG(s,b/\sqrt{d})$, with~$b >0 $ and $s \in \N^*$, then, for any $\varepsilon > 0$, taking $\delta = \varepsilon / d$,   we have that there exists a time $T > 0$, such that 
    \begin{align}
       |1 - m_T| \leq \varepsilon, \text{ and }\quad |T - \tau^+_{1/2}| \leq \mathsf{K} d \log(1/\varepsilon)\varepsilon^{-1} 
    \end{align}
    with probability at least $1 - \mathsf{K}\varepsilon$, for some generic $\mathsf{K} >0$ that depends solely of the link function $\phi$.
\end{theorem}
As introduced above, the important messages conveyed by this theorem are that (i) there is no difference between the different parameters setups captured by the information exponent $s$, and (ii) the time it takes to reach an $\varepsilon$-vicinity of $\theta^*$ is always strictly smaller than the one needed to exit the \textit{weak recovery phase} (e.g. $d$ compared to $d^{s-1}$ when $s \geq 3$). This means that \textit{the dynamics spends most of its time escaping the mediocrity}. Remark that we decided to present Theorem~\ref{thm:strong_recovery} resetting the step-size $\delta$ to put emphasis on the intrinsic difference between the two phases. Yet, we could have kept the same stepsize as in the weak recovery case: this obviously would slow down unnecessarily the second phase.

\begin{figure}
\begin{center}
\begin{tikzpicture}
\draw (-4,0) arc (180:360:4 and 1);
\draw [dashed] (-4,0) arc (180:0:4 and 1);
\draw (-3.95,.5) arc (180:360:3.95 and 1);
\draw [dashed] (-3.95,.5) arc (180:0:3.95 and 1);
\draw [thick] (-4,0) arc (180:0:4 and 4);

\draw [teal, thick] plot [smooth, tension=0.5] coordinates {
(2,-1.25,-1.75) (2.0653304520112862, -1.0432818908769334, -1.7570889677693344) (1.653304520112862, -1.35432818908769334, -1.7570889677693344) (1.53304520112862, -1.15432818908769334, -1.7570889677693344) (1.653304520112862, -1.15432818908769334, -1.7570889677693344) (1.653304520112862, -1.15432818908769334, -1.7570889677693344) (1.53304520112862, -1.15432818908769334, -1.7570889677693344) (1.353304520112862, -1.5432818908769334, -1.7570889677693344) (1.253304520112862, -1.15432818908769334, -1.7570889677693344)  (-0.8643467206473194, -1.2741589636098296, -1.713346571687929) (-0.5854241939363898, -1.5603375176592302, -1.7026887285749814) (-3.10167004625995113, -1.0207262123095036, -1.7403753036061693) (-4.0015938188224307037, -1.076955164501522, -1.797513647919381) (0.03958282309186134, -1.4126458731857201, -1.8082986090450741) (0.4304065824253258, -1.5714797424606084, -1.7826566209500503) (-0.18290568047143796, -1.3714851280644934, -1.665058121198839) (-0.2141140616704278, -1.262736795657368, -1.7735105695338182) (-0.7359662339324988, -1.1933133412385735, -1.7500513117700844) (-1.080017844455942, -1.4440976104037713, -1.7202310698906578) (-0.10790460267119584, -1.36709352843529, -1.7694202444428277) (-0.5717959946699571, -1.6709912933384345, -1.7287576853165842) (-1.0810227730873538, -1.2678003390731818, -1.7070830602634035) (-1.0534246646049377, -1.1545710992352378, -1.7360956651453512) (-0.7765052753289565, -1.289319675143908, -1.7077561390856055) (-3.3548928604060262, -1.1867863636986526, -1.781372913021331) (-4.29110656221205145, -0.6034746092966507, -1.7284470803877188) (-2.8352886387760792, -1.4516559311541479, -1.7389164213779167) 
(-1.86715707903012986, -1.132604490560212, -1.7959886616293295)  (-2.36715707903012986, -0.132604490560212, -1.7959886616293295) (-2.36715707903012986, 1.132604490560212, -1.7959886616293295) (-1.96715707903012986, 2.132604490560212, -1.7959886616293295) (-0.96715707903012986, 3.132604490560212, -1.4959886616293295) (0, 4.0, 0)
};

\draw (-1.86715707903012986, -1.42604490560212, -1.7959886616293295) node[anchor=south] () {\large \color{black} \textbullet};

\draw (-1.46715707903012986, -0.802604490560212, -1.7959886616293295) node() {\large \color{black} $\theta_{\mathsf{w}}$};

\draw [thick, thick, ->] (0,0) -- ++(0,4);
\draw (0.25,4.35) node () {\large $\theta^*$};

\draw [thick, thick, ->] (0,0,0) -- ++(2,-1.25,-1.75);
\draw (2.2,-1.2,-2) node () {\large $\theta_0$};

\draw [black, thick, <->] (0,0) ++(172.5:4.25cm) arc (172.5:180:4.25cm) node [midway, left] () {$b/\sqrt{d}\ $};
\end{tikzpicture}
\caption{Sketch of the SGD dynamics. After a long time spent in a band of typical size $1/\sqrt{d}$, the dynamics escapes at \textit{weak recovery} point $\theta_{\mathsf{w}}$ and then goes rapidly to the ground-truth $\theta^*$.}
\label{fig:nice_tikz}
\end{center}
\end{figure}
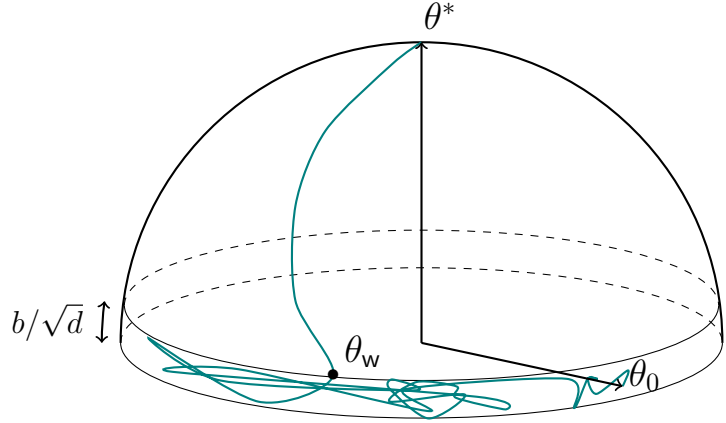

\section{Typical cases of loss landscape with \texorpdfstring{$\LPG$}{property} property}
\label{sec:population_landscape}

In this section, we showcase two prototypical cases where the $\LPG$ holds true: the section \ref{subsec:sym} deals with the spherically symmetric setting, whereas the section \ref{subsec:nonsym} describes a perturbative regime where the distribution is \textit{approximately} Gaussian is a quantitative sense.

\subsection{The symmetric case}
\label{subsec:sym}
We start our analysis with the spherically symmetric setting. We show that a spherical harmonic decomposition provides a valid extension of the Hermite decomposition in the Gaussian case, leading to essentially the same quantitative performance up to constant (in dimension) factors. 

\paragraph*{Spherical Harmonic Representation of the Population Loss.}
 We express the data distribution $\nu$ as a mixture of uniform measures $\nu = \int_0^\infty \tau_{r,d}\, \rho(dr) $, where $\tau_{r,d}=\mathrm{Unif}(r \S)$. 
Let $\tau_d = \tau_{1,d}$ and $u_d\in \mathcal{P}([-1,1])$ be the projection of $\tau_d$ onto one direction, with density given in close form by $u_d(dt) = Z^{-1}(1-t^2)^{(d-3)/2} \mathbf{1}(|t|\leq 1) dt $, where $Z$ is a normarlizing factor. 
Let $\{P_{j,d}\}_{j \in \mathbb{N}}$ be the orthogonal basis of Gegenbauer polynomials  of $L^2_{u_d}([-1,1])$, normalized such that $P_{j,d}(1) = 1$ for all~$j, d$. For each $r>0$, consider 
    $$l_r(\theta) := \langle\phi_\theta, \phi_{\theta^*}  \rangle_{\tau_{r,d}}  =   \langle \phi^{(r)}_\theta, \phi^{(r)}_{\theta^*}\rangle_{\tau_{d}}~,$$
where we define $\phi^{(r)}: [-1,1] \to \mathbb{R}$ such that $\phi^{(r)}(t):= \phi( r t)$. We write its decomposition in~$L^2_{u_d}([-1,1])$ as 
    $\phi^{(r)} = \sum_{j} \alpha_{j,r,d} P_{j,d}~,~\text{with }~\alpha_{j,r,d} = \frac{\langle \phi^{(r)}, P_{j,d} \rangle}{\| P_{j,d} \|^2}~.$
    Let $\Omega_{d}$ be the Lebesgue measure of $\S$, and $N(j,d) = \frac{2d+j-2}{d}\binom{d+j-3}{d-1}$ the so-called \textit{dimension} on the spherical harmonics of degree $j$ in dimension $d$. 
  From the Hecke-Funk representation formula \cite[Lemma 4.23]{frye2012spherical} and the chosen normalization $P_{j, d} (1) = 1$, we have $\| P_{j,d} \|_2 = \left(\frac{\Omega_{d-1}}{\Omega_{d-2} N(j,d)} \right)^{1/2}$ \cite[Proposition 4.15]{frye2012spherical} and obtain finally
  \begin{align}
      l_r(\theta) &=  
      \sum\nolimits_{j} \bar{\alpha}_{j,r,d}^2 P_{j,d}( \theta \cdot \theta^*)~,
  \end{align}
where we defined for convenience $\bar{\alpha}_{j,r,d} = \alpha_{j,r,d} / \sqrt{N(j,d)}$. As a result, it follows that the overall loss writes as solely the correlaton $m_\theta = \theta \cdot \theta^*$ as  
    \begin{align}
    \label{eq:lossmain}
        L(\theta) 
        = 2\|\phi\|^2_{L^2_\nu} -2 \sum\nolimits_{j} {\beta}_{j,d} P_{j,d}(m_\theta)~
        := \ell(m) ~,
    \end{align}
    where ${\beta}_{j,d} = \int_0^\infty \bar{\alpha}_{j,r,d}^2 \rho(dr) \geq 0$. 
Unsurprisingly, we observe that, thanks to the spherical symmetry and analogous to the Gaussian case, the loss still admits~$m=\theta \cdot \theta^*$ as a summary statistics. Yet, it is represented in terms of Gegenbauer polynomials, rather than monomials as in the Gaussian case. The monomial representation is a consequence of the stability of the Gaussian measure, as seen by the fact that 
$\langle \phi_\theta, \phi_{\tilde{\theta}} \rangle_{\gamma_d} = \langle \phi, \mathsf{A}_{\theta \cdot \tilde{\theta}} \phi \rangle_{\gamma}~,$
where $(\mathsf{A}_m f)(t) = \mathbb{E}_{z \sim \gamma} [f ( mt + \sqrt{1-m^2}z)]$ has a \emph{semi-group} structure (it is even known in fact as the Ornstein-Ulhenbeck semi-group). 

Let $\eta \in \mathcal{P}(\mathbb{R})$ be the marginal of $\nu$ along any direction. The following proposition gives a closed form formula of the coefficients $\beta_{j,d}$, represented as integrals over the radial distribution $\rho$ and projections of the link function $\phi$:
\begin{proposition}[Loss representation]
\label{prop:basicdec}
The $\beta_{j,d}$ defined in \eqref{eq:lossmain} have the integral representation
    \begin{equation}
        \beta_{j,d} = \langle \phi, \mathcal{K}_j \phi \rangle_{L^2_\eta} ~,
    \end{equation}
    where $\mathcal{K}_j$ is a positive semi-definite integral operator of $L^2_\eta$ that depends solely on $\rho$ and $\phi$.     
\end{proposition}
Note that a closed form expression of $\mathcal{K}_j$ can be found in Appendix \ref{secapp:basicdec}. The above proposition is in fact the stepping stone to calculate properly the \textit{information exponent} that plays a crucial role in the property $\LPG$. This is given through the link between the spectrum $\beta_{j,d}$ and the decomposition of $\phi$ in the $L^2_\eta$ orthogonal basis of polynomials, that we denote by $\{ q_j\}_j$. 
\begin{proposition}
\label{prop:infoexpo_sphere}
    Let $s = \inf\{ j; \beta_{j,d} > 0\}$ and $\tilde{s} = \inf\{ j; \langle \phi, q_j\rangle_{\eta} \neq 0\}$. Then $s \leq \tilde{s}$. 
\end{proposition}
Thus, the number of vanishing moments of $\phi$ with respect to the data marginal $\eta$ provides an upper bound on the `effective' information exponent of the problem $s$, as we will see next. 

\paragraph{Local Polynomial Growth.} 
From (\ref{eq:lossmain}), and as
$\nabla_\theta^{\mathcal{S}} L(\theta) = \ell'(m) (\theta^* - m \theta)~,$ we directly obtain
\begin{equation}
-\nabla_\theta^{\mathcal{S}} L(\theta) \cdot \theta^* = -(1 - m^2) \ell'(m) =  2 (1 - m^2) \sum\nolimits_{j} \beta_{j,d} P'_{j,d}(m)~,    
\end{equation}
%
which is the quantity we want to understand to exhibit the property $\LPG$ in this case. Hence, we now turn into the question of obtaining sufficient guarantees on the coefficients $(\beta_{j,d})_j$ that ensure local polynomial growth.
Since the typical scale of initialization for $m$ is $\Theta(1/\sqrt{d})$, 
our goal is to characterize sufficient conditions of local polynomial growth with $b = O(1/\sqrt{d})$. 

For that purpose, let us define two key quantities of Gegenbauer polynomials:
\begin{align}
    \upsilon_{j,d} & := -\min_{t \in (0,1)} P_{j,d}(t)~, ~~&\text{(smallest value)}\\
    z_{j,d} & := \arg\max\left\{ t \in (0,1); P_{j,d}(t)=0 \right\}~.~~&\text{(largest root)} 
\end{align}
We have the following sufficient condition based on the spectrum $(\beta_{j,d})_j$:
\begin{restatable}[Spectral characterization of $\LPG$]{proposition}{sufficientcond}
\label{prop:sufficientcond}
    Suppose there exist constants $K,C>0$ and $s \in \mathbb{N}$ such that we both have $\beta_{s,d} \geq C$ and $\sum_{j > s} \beta_{j,d} j (j+d-2)  \upsilon_{j-1,d+2} \leq K d^{(3-s)/2}$~.
    Then, taking $s^*$ as the infimum of such $s$, $L$ has the property $\LPG(s^*-1,z_{s^*,d})$. In particular, whenever $s^*\ll d$, we have $z_{s^*,d} \leq 2\sqrt{s^*/d}$.     
\end{restatable}
This proposition thus establishes that, modulo a mild regularity assumption expressed though the decay of the coefficients $\beta_{j,d}$, the spherically symmetric non-Gaussian setting has the same geometry as the Gaussian setting, for correlations \emph{slightly} above the equator. Crucially, the required amount of correlation to `feel' the local polynomial growth is a $O(\sqrt{s})$ factor from the typical initialization, and can be thus obtained with probability $\simeq e^{-s}$ over the initialization, according to Lemma~\ref{lem:initialization}, a lower bound which is \emph{independent of $d$}. 


The sufficient condition for $\beta_{j,d}$ appearing in Proposition \ref{prop:sufficientcond} involves the minimum values $\upsilon_{j,d}$ of Gegenbauer polynomials $P_{j,d}$, as well as sums of the form $\sum_j j^2 \beta_{j,d}$. 
In order to obtain a more user-friendly condition, we now provide an explicit control of $\upsilon_{j,d}$, and leverage mild regularity of $\phi$ to control $\sum_j j^2 \beta_{j,d}$. This motivates the following assumption on $\phi$ and $\nu$:
\begin{assumption}
\label{ass:phi_symmetric_assumptions}
    The link function $\phi$ satisfies $\phi \in L^2_\eta$ and $\phi' \in L^4_\eta$, and the radial distribution $\rho$ has finite fourth moment $\mathbb{E}_\rho [r^4] < \infty$. 
\end{assumption}
\begin{theorem}[$\LPG$ for symmetric distributions]
\label{coro:symmcase}
    Assume that $\phi$ and $\nu$ satisfy Assumption \ref{ass:phi_symmetric_assumptions}, and let $s^* = \inf\{j; \langle \phi, q_j \rangle_\eta \neq 0 \}$. Then $L$ has the property $\LPG(s^*-1,2\sqrt{s^*/d})$. 
\end{theorem}
The proof is provided in Appendix \ref{sec:proofthmsym}. At a technical level, the main challenge in proving Theorem \ref{coro:symmcase} is to achieve a uniform control of $\upsilon_{j,d}$ in $j$, a result which may be of independent interest. We address it by combining state-of-the-art bounds on the roots of the Gegenbauer polynomials, allowing us to cover the regime where $j$ is small or comparable to $d$, together with integral representations via the Cauchy integral formula, providing control in the regime of large $j$. On the other hand, we relate the sum $\sum_j j^2 \beta_{j,d}$ to a norm of $\phi'$ using a Cauchy-Schwartz argument, where we leverage the fourth moments from Assumption \ref{ass:phi_symmetric_assumptions}. 
\begin{remark}
Since we are in a setting where $\phi$ is known, an alternative to the original recovery problem from Eq (\ref{eq:mainsetup}) is to consider a pure Gegenbauer `student' link function of the form $\tilde{\phi} = P_{s,d}$, where $s$ is the information exponent from Proposition \ref{prop:infoexpo_sphere}. Indeed, the resulting population loss 
$\tilde{L}(\theta) = \mathbb{E} [ (\tilde{\phi}(x \cdot \theta) - \tilde{\phi}( x \cdot \theta^*) )^2]$ satisfies the $\LPG$ property, as easily shown in Fact \ref{fact:gegenwellbeh}.
\end{remark}


For the sake of completeness, we describe more precisely two concrete case studies below.


\begin{example}[Uniform Measure on the Sphere]
When $\nu = \text{Unif}(\sqrt{d} \S)$, we have $\rho = \delta_{\sqrt{d}}$, and therefore $\beta_{j,d} = \bar{\alpha}^2_{j, \sqrt{d},d}$. 
In that case, the orthogonal polynomial basis $\{q_j(t)\}_j$ of $L^2_\eta$ coincides with the rescaled Gegenbauer polynomials, $q_j(t) = P_{j,d}(t/\sqrt{d}) $.
Consider now a link function 
$\phi$ with $s-1$ vanishing moments with respect to $L^2_\eta$, i.e. such that 
$\bar{\alpha}_{j,d}=\langle \phi, q_j \rangle_\eta = 0$ for $j < s$ and $\bar{\alpha}_{s,d}=\langle \phi, q_s \rangle_\eta \neq 0$; and with sufficient decay in the higher harmonics as to satisfy the bound on the sum presented in Proposition~\ref{prop:sufficientcond} (for example, $\phi(t) = q_s(t)$ trivially satisfies this condition). 
Then Proposition \ref{prop:sufficientcond} applies and we conclude that the resulting population landscape satisfies $\LPG(s-1, O(\sqrt{s/d}))$.
\end{example}


In \cite{yehudai2020learning,wu2022learning} it is shown that monotonically increasing link functions\footnote{or link functions where their monotonic behavior dominates; see \cite{wu2022learning}.} lead to a benign population landscape, provided the data distribution $\nu$ satisfies mild anti-concentration properties. 
We verify that in our framework.
\begin{example}[Non-decreasing $\phi$]
    Indeed, Proposition \ref{prop:sufficientcond} is verified with $s=1$, provided $\phi' \in L^2_\eta$. Indeed, if $\phi\neq 0$ is monotonic, then we have
$\beta_1 = \langle \phi, \mathcal{K}_1 \phi \rangle_\eta = C_d \left(\mathbb{E}_{\eta} [t \phi(t)] \right)^2 \neq 0~,$
since we can assume without loss of generality that $\phi(t)\geq 0$ for $t\geq 0$ and $\phi(t) \leq 0$ for $t \leq 0$. 
\end{example}
We emphasize that the results of \cite{yehudai2020learning,wu2022learning} extend beyond the spherically symmetric setting, which is precisely the focus of next section.

\subsection{Non-Spherically Symmetric Case}
\label{subsec:nonsym}
We now turn to the setting where $\nu$ is no longer assumed to have spherical symmetry. 
By making further regularity assumptions on $\phi$, our main insight is that distributions that are \emph{approximately} symmetric (defined in an appropriate sense) still benefit from a well-behaved optimization landscape. 

\paragraph*{Two-dimensional Wasserstein Distance.}
When $\nu$ is not spherically symmetric, the machinery of spherical harmonics does not apply, and we thus need to rely on another structural property. Consider a centered and  isometric data distribution $\nu \in \mathcal{P}_2(\mathbb{R}^d)$, i.e. such that $\mathbb{E}_\nu x=0$ and $\Sigma_\nu=\mathbb{E}_\nu [x x^\top] = I_d$. 
We consider the two-dimensional $1$-Wassertein distance \cite[Definition 1]{niles2022estimation} --see also \cite{paty2019subspace}-- between a pair of distributions $\nu_a, \nu_b \in \mathbb{P}(\mathbb{R}^d)$, defined as 
    \begin{equation}
        \widetilde{W}_{1,2}(\nu_a, \nu_b) := \sup_{P \in \mathrm{Gr}(2, d)} W_1( P_\# \nu_a, P_\# \nu_b)~,
    \end{equation}
    where the supremum runs for any two-dimensional subspace $P \in \mathrm{Gr}(2, d)$, and $P_\# \nu \in \mathcal{P}(\mathbb{R}^2)$ is the projection (or marginal) of $\nu$ onto the span of $P$. 
$\widetilde{W}_{1,2}$ is a distance (\cite[Proposition 1]{paty2019subspace}) and measures the largest $1$-Wasserstein distance between any two-dimensional marginals. 

We are in particular interested in the setting where $\nu_a = \nu$ is our data distribution, and $\nu_b$ is a reference symmetric measure --  for instance the standard Gaussian measure $\gamma_d$. Consider the fluctuations 
\begin{align}
\Delta_L(\theta) & := |L(\theta) - \bar{\ell}(m_\theta)| \text{ and }~,\\
\Delta_{\nabla L}(\theta) &:= | \nabla^{\mathcal{S}}_\theta L(\theta) \cdot \theta^* - \bar{\ell}'(m_\theta) (1 - m_\theta^2) |~,
\end{align}
where $\bar{\ell}(m)$ is the Gaussian loss defined in Section \ref{sec:preliminaries} and $L(\theta) = \mathbb{E}_\nu [ | \phi(  x \cdot \theta ) - \phi(  x \cdot \theta^* )|^2]$.
$\Delta_L$ and $\Delta_{\nabla L}$ thus measure respectively the fluctuations of the population loss and the relevant (spherical) gradient direction. 
By making additional mild regularity assumptions on the link function $\phi$, we can obtain a uniform control of the population loss geometry using the dual representation of the $1$-Wasserstein distance. 

\begin{restatable}[Regularity of link function]{assumption}{asslipreg}
\label{ass:lipreg}
    We assume that $\phi, \phi'$ are both $B$-Lipschitz, 
    and that $\phi''(t) = O(1/t)$. 
\end{restatable}
\begin{restatable}[Subgaussianity]{assumption}{asssubgaussian}
\label{ass:subgaussian}
    The data distribution $\nu$ is $M$-subgaussian: for any~$v \in \S$, we have $\|  x\cdot v  \|_{\psi_2} \leq M $, where $\|z\|_{\psi_2}:= \inf\{ t>0; \, \mathbb{E}[\exp(z^2/t^2)\leq 2\}$ is the Orlitz-2 norm.        
\end{restatable}

\begin{restatable}[Uniform gradient approximation]{proposition}{gradwassers}
    \label{prop:gradwassers}
    Under Assumptions \ref{ass:lipreg} and \ref{ass:subgaussian}, for all~$\theta \in \S$, 
\begin{equation}
     \Delta_{\nabla L}(\theta) = (1-m^2) O\left( \widetilde{W}_{1,2}(\nu, \gamma) \log ( \widetilde{W}_{1,2}(\nu, \gamma)^{-1} ) \right)
\end{equation}
where the $O(\cdot)$ notation only hides constants appearing in Assumptions \ref{ass:lipreg} and \ref{ass:subgaussian}.   
\end{restatable}

In words, the population gradient under $\nu$ is viewed as a perturbation of the population
gradient under $\gamma$, which has the well-behaved geometry already described in Section~\ref{sec:preliminaries}. These perturbations 
can be uniformly controlled by the projected 1-Wasserstein distance, thanks to the subgaussian tails of $\nu$. 

Our focus will be in situations where $\widetilde{W}_{1,2}(\nu, \gamma)=O(1/\sqrt{d})$. This happens to be the `natural' optimistic scale for this metric in the class of isotropic distributions $\Sigma_\nu = I_d$, as can be seen for instance when $\nu = \mathrm{Unif}(\sqrt{d} \S)$. Under such conditions, it turns out that link functions with information exponent $s\leq 2$ can be recovered with simple gradient-based methods, by paying an additional polynomial (in $d$) cost in time complexity. 

\begin{assumption}
\label{ass:smallinfoexponent}
    The Gaussian information exponent of $\phi$, $s:=\arg\min\{ j; \langle \phi, H_j \rangle  \neq 0\}$ satisfies $s \leq 2$. 
\end{assumption}
\begin{assumption}
\label{ass:wassersteinsmall} The projected Wasserstein distance satisfies
    $\widetilde{W}_{1,2}(\nu, \gamma) \leq M'/\sqrt{d}$.
\end{assumption}
\begin{restatable}[$\LPG$, non-symmetric setting]{proposition}{propgrad_expo2}
\label{prop:grad_expo2}
    Under Assumptions \ref{ass:lipreg}, \ref{ass:subgaussian}, \ref{ass:smallinfoexponent} and \ref{ass:wassersteinsmall},~$L$ verifies $\LPG\left(1,O\left(\sqrt{\frac{(\log d^\kappa)}{d}}\right)\right)$, where $\kappa$ depends only on $B,M,M'$. 
\end{restatable}
This proposition illustrates the cost of breaking spherical symmetry in two aspects: (i) it requires additional regularity on $\phi$, and notably restricts its (Gaussian) information exponent to $s=2$, and (ii) the scale to reach $\LPG$ is now no longer dimension-free, but has a polynomial dependency on dimension, since from Lemma \ref{lem:initialization}, picking $\delta$ any positive constant we have 
$$\P_{\theta_0}\left(m_{\theta_0} \geq \sqrt{\log d^\kappa}/\sqrt{d}\right) \geq \frac{\delta}{4}  e^ {- (\sqrt{\log d^\kappa} + \delta)^2} \geq \Omega\left( d^{-(1+o(1))\kappa}\right)~. $$
At present, we are not able to rule this out as a limitation of our proof;  establishing whether this polynomial dependency on dimension is an inherent cost of the symmetry breaking is an interesting question for future work. 

While assumptions \ref{ass:lipreg}, \ref{ass:subgaussian} and \ref{ass:smallinfoexponent} are transparent and impose only mild conditions on the link function and tails of $\nu$, the `real' assumption of Proposition \ref{prop:grad_expo2} is the concentration of $\widetilde{W}_{1,2}(\nu, \gamma)$ (Assumption \ref{ass:wassersteinsmall}). The ball $\{ \nu; \widetilde{W}_{1,2}(\nu, \gamma) = O(1/\sqrt{d}) \}$ contains many non-symmetric measures, for instance empirical measures sampled from $\gamma$ with $n = \omega(d^2)$ \cite[Proposition 8]{niles2022estimation}, and we suspect it contains many other examples, such as convolutions of the form $\nu \ast \gamma_{\sigma}$ arising for instance in diffusion models. That said, one should \emph{not} expect the distance $\widetilde{W}_{1,2}(\nu, \gamma)$ to be of order $1/\sqrt{d}$ for generic `nice' distributions $\nu$; for instance, log-concave distributions are expected to satisfy $W_1( P_\# \nu, \gamma) \simeq 1/\sqrt{d}$ for \emph{most} subspaces $P$, as captured in Klartag's CLT for convex bodies \cite{klartag2007central}. In summary, many situations of interest fall outside this regime, which motivates us to relax the uniform Wasserstein criterion. 


\paragraph{Localized Growth via Stein's method.}
To illustrate the mileage of the previous techniques beyond this `quasi-symmetric' case, we consider now an idealised setting where the data is drawn from a product measure $\nu = \eta^{\otimes d}$, with $\eta \in \mathcal{P}(\mathbb{R})$ and $W_1( \eta, \gamma_1) = \Theta(1)$. In other words, $x=(x_1, \ldots, x_d) \sim \nu$ if $x_i \sim \eta$ are i.i.d. In this setting, the distances $W_1( P_\# \nu, \gamma)$ reflect a CLT phenomena, which requires the subspace $P$ to `mix' across independent variables. Consequently, one may expect the expression of the hidden direction $\theta^*$ in the canonical basis to play a certain role.
For that purpose, we make the following additional regularity assumption on the tails of $\phi$ to simplify the quantitative bounds: 
\begin{restatable}[Additional Regularity in third derivatives] {assumption}{assthird}
\label{ass:extralip}
    $\phi$ admits four derivatives bounded by $L$, with $|\phi^{(3)}(t)| = O(1/t)$ and $|\phi^{(4)}(t)| = O(1/t^2)$. Moreover, the third moment of the data distribution is finite:
    $\tau_3=\mathbb{E}_{t \sim \eta}[t^3] < \infty$.    
\end{restatable}
Stein's method provides a powerful control on $\Delta_L(\theta)$ and 
$\Delta_{\nabla L}(\theta)$, as shown by the following result:
\begin{restatable}[Stein's method for product measure]{proposition}{propstein}
\label{prop:stein}
    Let $\chi(\theta, \theta^*):= \| \theta\|_4^2 + \| \theta^*\|_4^2$. 
    Under Assumptions \ref{ass:lipreg}, \ref{ass:subgaussian} and \ref{ass:extralip}, there exists a universal constant $C=C(M, B, \tau_3)$ such that 
        \begin{equation}
        \label{eq:loss_conc}
        \Delta_L(\theta) \leq C \chi(\theta, \theta^*)~,\text{and }~\Delta_{\nabla L}(\theta) \leq C \sqrt{1-m^2} \chi(\theta, \theta^*)~.
    \end{equation}    
\end{restatable}
The proof is based on the Stein's method for multivariate variables \cite[Theorem 3.1]{rollin2013stein} with independent entries, which provides a quantitative CLT bound. Contrary to the quasi-symmetric case, here the concentration is not uniform over the sphere, but crucially depends on the sparsity of \emph{both} $\theta$ and $\theta^*$, measured via the $\ell_4$ norms $\| \theta \|_4$, $\| \theta^* \|_4$: for incoherent, non-sparse directions, we have $\|\theta \|_4^2 \simeq 1/\sqrt{d}$, recovering the concentration rate that led to Proposition \ref{prop:grad_expo2},  
while for sparse directions we have $\| \theta\|_4^2 = \Theta(1)$, indicating an absence of concentration to the Gaussian landscape.

Therefore, the natural conclusion is to assume a planted model where $\theta^*$ is \emph{incoherent} with the data distribution, i.e. $\| \theta^* \|_4 = O(d^{1/4})$. While the $\LPG$ property does not directly apply in this setting, we outline an argument that suggests that the single-index model can still be efficiently solved using gradient-based methods. 
For that purpose, we assume that $\theta^*$ is drawn uniformly in $\S$, which implies that its squared-$L^4$ norm $\| \theta^*\|_4^2$ is of order $d^{-1/2}$ with high probability:
\begin{fact}
    Assume $\theta^* \sim \mathrm{Unif}(\S)$. Then $\mathbb{P}( \| \theta^*\|_4^2 \leq C /\sqrt{d}) \geq 1 - C'\exp(-C)$. 
\end{fact}

Because $\theta_0$ is also drawn uniformly on the sphere, the typical value of $\chi(\theta, \theta^*)$ is of order $d^{-1/2}$. For Gaussian information exponent $s=2$, the population gradient under Gaussian data satisfies $-\nabla_\theta^{\S}L_\gamma(\theta) \cdot \theta^* \geq C m_\theta$. As a consequence, by Proposition \ref{prop:stein}, whenever $|\theta_0 \cdot \theta^*| > c \sqrt{d}$ (which happens with constant probability lower bounded by $e^{-c}$), we enter a `local' $\LPG$ region where $-\nabla_\theta^{\S}L(\theta) \cdot \theta^* \geq C (m - c/\sqrt{d}) > 0$. 
While this condition is sufficient in the quasi-symmetric setting to start accumulating correlation (Theorem \ref{thm:weak_recovery}), now this event is conditional on $\theta$ being dense, ie so that $\chi(\theta, \theta^*) = O(1/\sqrt{d})$. 

Since the typical value of $\chi(\theta, \theta^*)$ is of scale $1/\sqrt{d}$, one would expect that SGD will rarely visit sparse points where $\chi(\theta, \theta^*)\gg O(1/\sqrt{d})$, and thus that the local $\LPG$ property will be valid for \emph{most} times during the entropic phase of weak recovery --- 
and therefore that the correlation $m_\theta$ will pile-up as in the quasi-symmetric setting. 

We summarise this property in the following conjecture:
\begin{conjecture}[SGD avoids sparse points]
\label{conj:SGDparse}
    Assume $\theta^*, \theta_0$ are drawn from the uniform measure, and let $\theta_t$ be the $t$-th iterate of SGD with $\delta \simeq 1/(d \log d)$.  
    There exists a universal constant $C$ such that for any $\xi>0$, we have 
    \begin{equation}
        \mathbb{P}\left( \sup_{t \leq T} \| \theta_t\|_4^2 \geq \sqrt{\frac{{\xi} \log T }{{d}}} \right) \leq C \exp(-\xi^2 d)~.
    \end{equation}
\end{conjecture}
Since the time to escape mediocrity in the case $s=2$ is $T \simeq d \log(d)^2$, this conjecture would imply that SGD does not effectively `see' any sparse points, and thus escapes mediocrity. If one assumed that in this phase the dynamics is purely noisy, now pretending that $\theta_i$ were drawn independently from the uniform measure, and that $\|\theta\|_4^4$ is approximately Gaussian with mean $d^{-1}$ and variance $d^{-3}$, the result follows by simple concentration. The challenging aspect of Conjecture \ref{conj:SGDparse} is precisely to handle the dependencies across iterates, as well as the spherical projection steps.






\section{Experiments}

In order to validate our theory, and inspect the degree to which our bounds may be pessimistic, we consider empirical evaluation of the training process in our two primary settings.  Specifically, we consider random initialization on the half-sphere (with the sign chosen to induce positive correlation as in \cite{arous2021online}), and investigate how often strong recovery occurs relative to the information exponent of the link function.

\begin{figure*}[ht]
\centering

\begin{subfigure}{.45\textwidth}
  \centering
  \includegraphics[width=1.\linewidth]{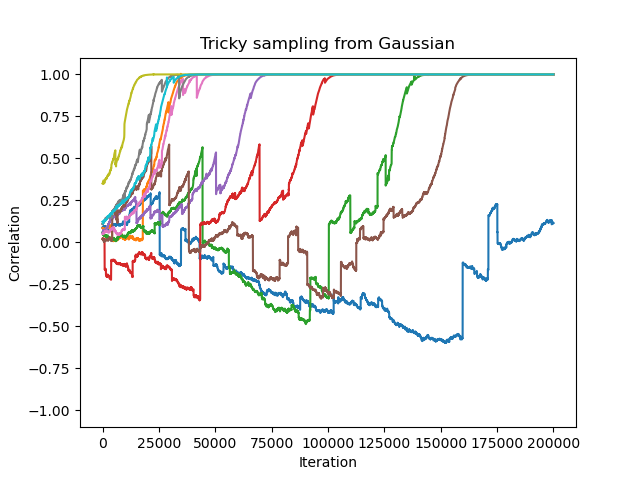}
\end{subfigure}%
\begin{subfigure}{.45\textwidth}
  \centering
  \includegraphics[width=1.\linewidth]{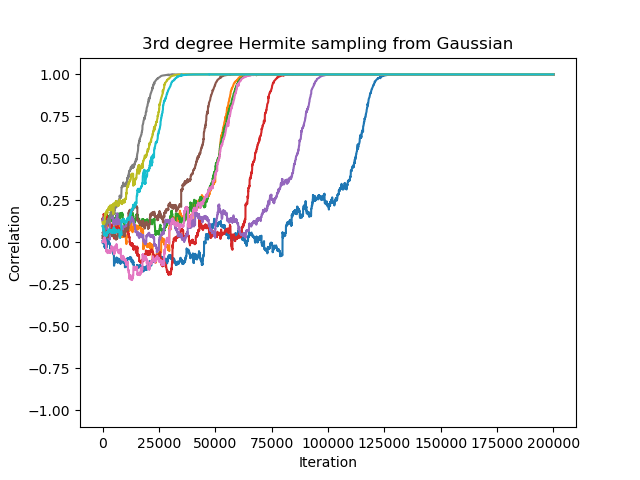}
\end{subfigure}%

\begin{subfigure}{.45\textwidth}
  \centering
  \includegraphics[width=1.\linewidth]{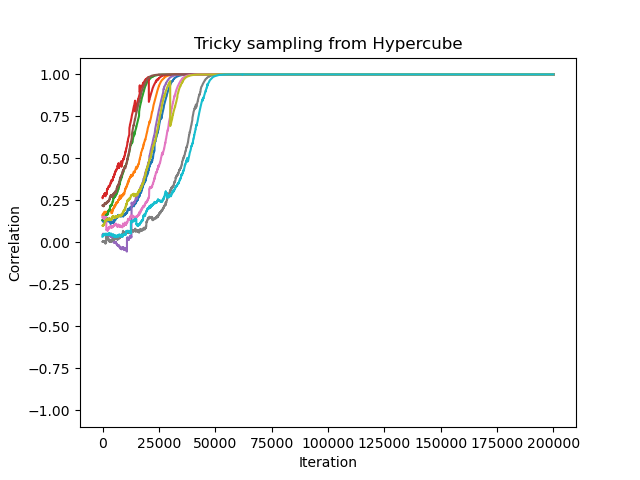}
\end{subfigure}%
\begin{subfigure}{.45\textwidth}
  \centering
  \includegraphics[width=1.\linewidth]{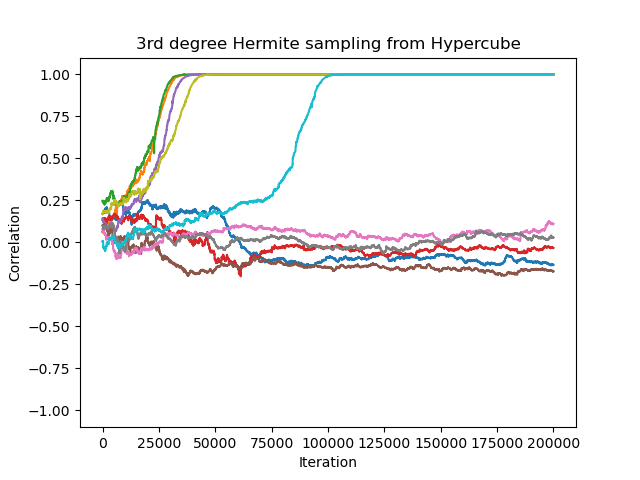}
\end{subfigure}%

\begin{subfigure}{.45\textwidth}
  \centering
  \includegraphics[width=1.\linewidth]{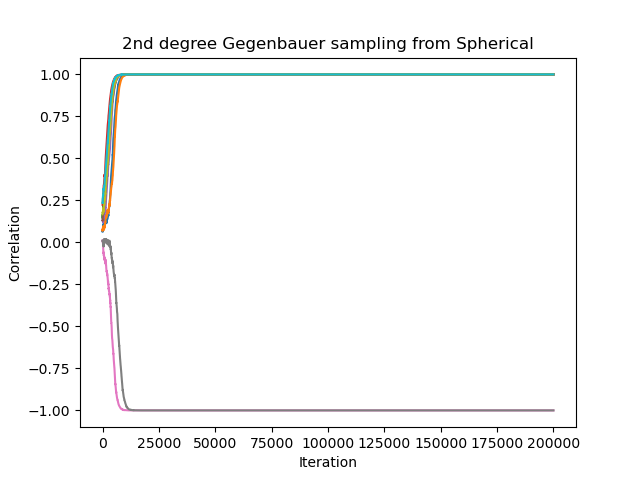}
\end{subfigure}%
\begin{subfigure}{.45\textwidth}
  \centering
  \includegraphics[width=1.\linewidth]{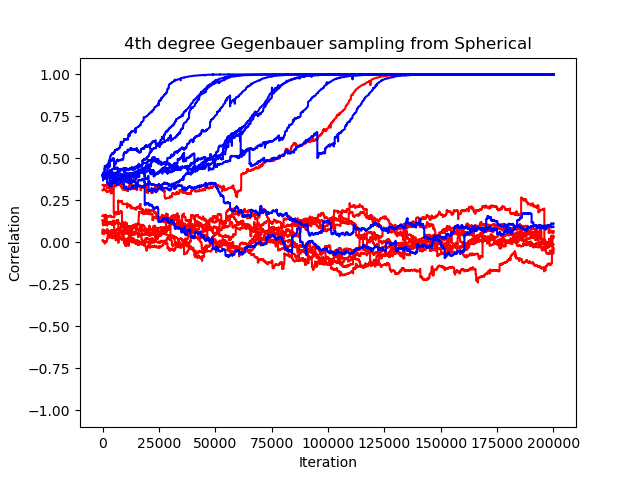}
\end{subfigure}%

\caption{Correlation with the true signal throughout training, under different choices of link function and input distribution.}
\vspace{-0.5cm}
\label{fig:runs}
\end{figure*}

\begin{figure*}[ht]
\centering

\begin{subfigure}{.4\textwidth}
  \centering
  \includegraphics[width=1.\linewidth]{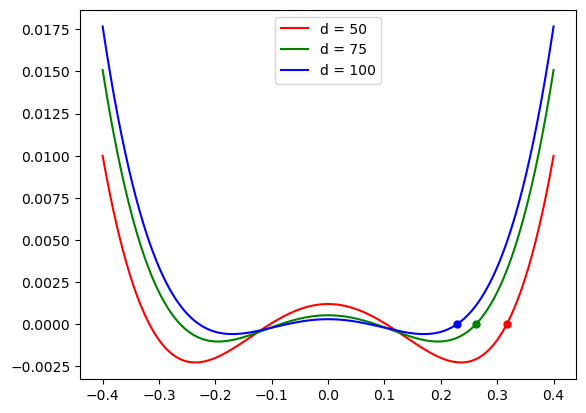}
\end{subfigure}%

\caption{The 4th degree Gegenbauer polynomial with dimension 50, 75, 100.  Each is equivalent (up to rescaling) to the loss landscape of learning the 4th degree Gegenbauer in the appropriate dimension.  Points indicate largest zeros.}
\label{fig:gegen}
\end{figure*}

\paragraph*{Symmetric Case}

For the spherically symmetric setting, we experiment with the input distribution that is uniform on the sphere.  We are primarily interested in verifying that, unlike the Gaussian case, strong recovery depends on whether the initial correlation is sufficiently high to avoid local minima and benefit from the $\LPG$ guarantee.  This is not evident in the 2nd degree Gegenbauer case, which is monotonic and quickly reaches strong recovery, but it is clear from the 4th degree Gegenbauer link function.

In the infinite sample setting, Figure~\ref{fig:gegen} exactly characterizes the loss landscape when learning the 4th degree Gegenbauer under inputs uniform on $\S$ for different values of $d$.  Note that the largest zero for $d = 50$ occurs at $\approx \pm 0.31$, and the loss is monotonic for $m$ values initialized outside that region.  This phenomenon persists for higher dimensions, and one may observe that $d$ increases, the critical points become smaller in magnitude, according to the scaling $\simeq \sqrt{1/d}$.

The bottom right subplot in Figure~\ref{fig:runs} indicates training runs in this setting, where red lines are initialized uniformly on the sphere, and blue lines are initialized uniformly conditioned on $m = 0.4$, which is slightly past the last zero of the polynomial.  We observe that random initialization infrequently exceeds the threshold necessary for strong recovery, but planting the initialization above this threshold gives a high probability of recovery.

\paragraph*{Non-Symmetric Case}

For the non-spherically symmetric setting, we compare the performance of Gaussian inputs with inputs that are approximately Gaussian under a two-dimensional projection.  For simplicity, we loosen our assumptions slightly, and consider the input distribution as the $d$ dimensional product distribution of uniform random variables (rescaled to have unit variance), and allow for a non-Lipschitz link function.  Here, we are primarily interested in whether Assumption~\ref{ass:smallinfoexponent} is tight and $s \leq 2$ is necessary for recovery, as well as whether Conjecture~\ref{conj:SGDparse} holds in practice.

To evaluate, we compare strong recovery rates when training on a "tricky" function with $s = 2$ (chosen to be $\frac{1}{2} \left(h_2 - h_3 - h_4 + h_5 \right)$) versus a function with $s = 3$ (simply the degree three hermite polynomial $h_3$).  We make this choice for the $s = 2$ function in order to produce a function which is not monotonic, for which learning is easy under many distributions, as discussed in~\cite{yehudai2020learning}.

In Figure~\ref{fig:runs} we observe that strong recovery reliably occurs for both the Gaussian and hypercube input distributions when $s = 2$.  There is more variance in the Gaussian runs, likely because the magnitude of the gradients will be larger due to the inclusion of high degree terms.  But for the $s = 3$ case, the Gaussian distribution converges quickly while the hypercube distribution frequently cannot escape the equator.

\section{Conclusion and Perspectives}

In this work, we have asked whether the remarkable properties of high-dimensional Gaussian SGD regression of single-index models are preserved as one loses some key aspects that make Gaussian distributions so special (and so appealing for theorists). 
Our results are mostly positive, indicating a robustness of the Gaussian theory, especially within the class of spherically symmetric distributions, where a rich spherical harmonic structure is still available.  As one loses spherical symmetry, the situation becomes more dire, motivating a perturbative analysis that we have shown is effective via projected Wasserstein and Stein couplings. 

That said, there are several open and relevant avenues that our work has barely touched upon, such as  understanding whether the robustness can be transferred to other algorithms beyond SGD, or addressing the semi-parametric problem when the link function is unknown, along the lines of \cite{biettilearning2022,abbe2023sgd,damian2022neural,berthier2023learning}. A particularly interesting direction of future work is to extend the analysis of product measures to `weakly dependent' distributions, motivated by natural images where locality in pixels captures most (but not all) of the statistical dependencies. Stein's method appears to be a powerful framework that can accommodate such weak dependencies, and deserves future investigation.

\paragraph{Acknowledgements} This work was partially supported by NSF DMS 2134216, NSF CAREER CIF 1845360, NSF IIS 1901091 and the Alfred P Sloan Foundation. 

\bibliographystyle{apalike}
\bibliography{references}

\clearpage

\appendix

$ $

{\Huge \textsc{Appendix}}

\vspace{0.5cm}

We gather in the appendix the proofs of the theorems, propositions and lemmas stated in the main text. In Section~\ref{sec:proof_lemma_1}, the reader will find a short proof of Lemma \ref{lem:initialization}. In Section~\ref{sec:sgdapp}, we prove Theorems~\ref{thm:weak_recovery} and \ref{thm:strong_recovery} on the SGD dynamics. Sections~\ref{sec:symapp} and~\ref{sec:nonsymapp} are respectively devoted to prove that the $\LPG$ property holds in some spherical symmetric case and under some perturbative regime.






\section{Proof of Lemma \ref{lem:initialization}}
\label{sec:proof_lemma_1}

Let us first recall the Lemma before writing a proof of it.

\begin{lemma*}
    For all $a > 0$, we have $ \P_{\theta_0}(m_{\theta_0} \geq a/\sqrt{d}) \leq a^{-1}e^{-a^2/4}  $. Additionally, for any $\delta>0$ such that $ \max\{a, \delta\} \leq \sqrt{d}/4$, we have the lower bound:
    $\P_{\theta_0}(m_{\theta_0} \geq a/\sqrt{d}) \geq \frac{\delta}{4} e^ {- (a + \delta)^2} $.
\end{lemma*}
\begin{proof}
By rotation invariance of the uniform distribution of the sphere, $m_{\theta_0}$ is distributed according to~$\theta_0[1]$, the first coordinate of the vector $\theta_0 \in \S$. By a particular case of Stam’s formula \cite[relation (3)]{stam1982limit}, we know that for $d \geq 3$, both are distributed according to the probability of density, $\forall  t \in \R$,
\begin{align*}
    \tau(t):=\frac{\Gamma(d/2)}{\sqrt{\pi} \Gamma((d-1)/2)} \left(1 - t^2\right)^{(d-3)/2} \mathds{1}_{[-1, 1]}.
\end{align*}
First, note that we can upper and lower bound the constant by the following:
\begin{align*}
   \sqrt{\frac{d}{3}} \leq \frac{\Gamma(d/2)}{\Gamma((d-1)/2)} \leq \sqrt{\frac{d}{2}},
\end{align*}
for $d \geq 6$ by \cite[equality 3.2]{laforgia2013some}, which was already proved in \cite{gautschi1959some}. 

Hence, in terms of the upper bound, we have:
\begin{align*}
    \P_{\theta_0}(m_{\theta_0} \geq a/\sqrt{d}) &\leq \sqrt{\frac{d}{2\pi}} \int_{a/\sqrt{d}}^1  \left(1 - t^2\right)^{(d-3)/2} dt \\
    &\leq \sqrt{\frac{d}{2 \pi}} \int_{a/\sqrt{d}}^1  e^{-\frac{d-3}{2} t^2} dt \\
    &\leq \frac{1}{\sqrt{2\pi}}\frac{d}{a} \int_{a/\sqrt{d}}^1 t  e^{-\frac{d-3}{2} t^2} dt \\
    &\leq \frac{1}{2a}   e^{-a^2/4}~,
\end{align*}
which concludes the first part of the result.

Second, let $a \leq \sqrt{d}/4$ and take any $0<\delta< \sqrt{d}/4$. We have, 
\begin{align*}
    \P_{\theta_0}(m_{\theta_0} \geq a/\sqrt{d}) &\geq \sqrt{\frac{d}{3\pi}} \int_{a/\sqrt{d}}^1  \left(1 - t^2\right)^{(d-3)/2} dt \\
    &\geq  \sqrt{\frac{d}{3\pi}} \int_{a/\sqrt{d}}^{(a + \delta)/\sqrt{d}}  \left(1 - t^2\right)^{(d-3)/2} dt \\
    &\geq \sqrt{\frac{d}{3\pi}} \frac{\delta}{\sqrt{d}} \left(1 - \frac{(a + \delta)^2}{d} \right)^{(d-3)/2},
\end{align*}   
where the last inequality simply comes from the fact that $t \to (1-t^2)^{(d-3)/2}$ is non-increasing. Going further, if we lower bound the term with the negative $-3/2$ power by $1$, we have
\begin{align*}
    \P_{\theta_0}(m_{\theta_0} \geq a/\sqrt{d}) &\geq \frac{\delta}{4} \exp\left( \frac{d}{2} \log\left(1 - \frac{(a + \delta)^2}{d} \right)\right) \\
    &\geq  \frac{\delta}{4} \exp\left( - \frac{(a + \delta)^2}{2 \left( 1 - (a + \delta)^2/d \right)} \right)~,
\end{align*}
where the last inequality come from the classical bound $\log (1 + x) \geq x/(1+x)$, for $x > - 1$. Furthermore, as, $(a+\delta)^2 \leq d/2$, we have finally
\begin{align*}
    \P_{\theta_0}(m_{\theta_0} \geq a/\sqrt{d}) &\geq  \frac{\delta}{4} e^ {- (a + \delta)^2},
\end{align*}
which finalizes the proof of the Lemma.
\end{proof}

\section{Proofs on the SGD dynamics: Section \ref{sec:sgd}}
\label{sec:sgdapp}

We first recall the notations useful to fully describe the dynamics. In Section~\ref{subsec:weak_recovery}, we prove Theorem~\ref{thm:weak_recovery} about weak recovery. In Section~\ref{subsec:strong_recovery}, we prove Theorem~\ref{thm:strong_recovery} about strong recovery. Finally, 

\subsection{Recalling the dynamics}

For the sake of clarity, let us recall the notations and facts developed in the main text. The overall loss classically corresponds to the average over all the \textit{data} of a square penalisation $l(\theta, x) = (\phi_\theta(x) - \phi_{\theta^*}(x))^2$ so that 
$$L(\theta) = \mathbb{E}_{\nu} [(\phi_\theta(x) - \phi_{\theta^*}(x))^2].$$
To recover the signal given by $\theta^*$, we run \textit{online stochastic gradient descent} on the sphere $\S$. This corresponds to have at each iteration $t \in \mathbb{N}^*$ a \textit{fresh sample} $x_t$ independent of the filtration $\mathcal{F}_t = \sigma(x_1, \dots, x_{t-1})$ and perform a spherical gradient step, with step-size $\delta>0$, with respect to $\theta \to l(\theta, x_t)$:
\begin{align}
    \label{eq:SGD_app}
    \theta_{t+1} &= \frac{\theta_t - \delta \nabla_\theta^\mathcal{S}  l(\theta_t, x_t)}{\left|\theta_t - \delta \nabla_\theta^\mathcal{S}  l(\theta_t, x_t)\right|},
\end{align}
initialized at $\theta_0$ uniformly on the sphere: $\theta_{0} \sim \mathrm{Unif}(\S)$. Recall that we use the notation $\nabla_\theta^\mathcal{S}$ to denote the spherical gradient, that is $$ \nabla_\theta^\mathcal{S} l(\theta, x) =  \nabla_\theta  l(\theta, x) - (\nabla_\theta  l(\theta, x) \cdot \theta) \theta.$$
Let us introduce the following frequently used notations: for all $t \in \mathbb{N}^*$, we denote the normalization by $r_t := r(\theta_t, x_t) = \left|\theta_t - \delta \nabla_\theta^\mathcal{S}  l(\theta_t, x_t)\right|$ and the martingale induced by the stochastic gradient descent as $M_t = M(\theta_t, x_t) = l(\theta_t, x_t) - \E_\nu [ l(\theta_t, x)]$.

\subsection{Tracking the correlation.} Recall that the relevant signature of the dynamics is the one-dimensional correlation: $m_t = \theta_t \cdot \theta^*$. Let us re-write the iterative recursion followed by $(m_t)_{t \geq 0}$, with the notation recalled above, for $t \in \mathbb{N}^*$, 
\begin{align}
\label{eq:dynamics_main_m_app}
m_{t+1} = \frac{1}{r_t}\left( m_t - \delta \nabla^S l(\theta_t, x_t) \cdot \theta^*  \right) =  \frac{1}{r_t}\left( m_t - \delta\nabla^S L(\theta_t) \cdot \theta^* - \delta\nabla^S M_t \cdot \theta^*  \right).
\end{align}
We want to lower bound the right hand side of~\eqref{eq:dynamics_main_m_app}. We begin by a lower bound on~$1/r_t$. 

\begin{lemma}[Bound on $r_t$]
    For all $t \in \N^*$, we have $1/r_t \geq 1 - \delta^2 \left|\nabla_\theta l(\theta_t, x_t)\right|^2$.    
\end{lemma}
\begin{proof}
    For all $t \in \N^*$, we have, by orthogonality of as $\theta_t$ and $\nabla_\theta^\mathcal{S}  l(\theta_t, x_t)$, that
    \begin{align*}
        r_t^2 = \left|\theta_t - \delta \nabla_\theta^\mathcal{S}  l(\theta_t, x_t)\right|^2  = 1 + \delta^2 \left|\nabla_\theta^\mathcal{S} l(\theta_t, x_t)\right|^2 \leq 1 + \delta^2 \left|\nabla_\theta l(\theta_t, x_t)\right|^2.
    \end{align*}
    Hence, from the inequality $ (1 + u)^{-1/2} \geq 1 - u $ for all $ u > 0$, we conclude the proof.
\end{proof}
Thanks the fact that $L$ satisfies $\LPG(s,b/\sqrt{d})$, ie $-\nabla^\mathcal{S}  L(\theta) \cdot \theta^* \geq C (1 - m)(m - b/\sqrt{d})^{s-1}$, we have that the dynamics satisfies the following inequality between iterates:
\begin{align}
\label{eq:dynamics_lower_m}
    m_{t+1} &\geq m_t + C \delta \left(1 - m_t\right) \left(m_t - \frac{b}{\sqrt{d}}\right)^{s-1} - \delta\nabla^S M_t \cdot \theta^* - \delta^2  |m_t| \left|\nabla_\theta l(\theta_t, x_t)\right|^2 -  \delta^3 \xi_t, 
\end{align}
where $\xi_t = \left|\nabla_\theta l(\theta_t, x_t)\right|^2 |\nabla^S l(\theta_t, x_t) \cdot \theta^* |$. All the terms of the inequality have a natural origin: the second term is the ideal term coming from the gradient flow and the growth condition, the third term corresponds to the martingale increments coming form the noise induced by SGD and the two final terms are simply discretization errors coming from discrete nature of the procedure and the projection step. 

However, to have a tight dependency with respect to the dimension, we need to be extra careful. This is why, following \cite{arous2021online}, we decompose this term introducing a threshold $\MM > 0$, to be fixed later, such that: 
\begin{align*}
    |m_t| \left|\nabla_\theta l(\theta_t, x_t)\right|^2 & =|m_t| \left|\nabla_\theta l(\theta_t, x_t)\right|^2 \mathds{1}_{\{\left|\nabla_\theta l(\theta_t, x_t)\right|^2 \leq \MM\}} + |m_t| \left|\nabla_\theta l(\theta_t, x_t)\right|^2\mathds{1}_{\{\left|\nabla_\theta l(\theta_t, x_t)\right|^2 > \MM\}}
\end{align*}
With the same notations and summing all these terms until time $T \in \N^*$, we can write 
\begin{align*}
    m_{T} \geq m_0 &+ C \delta \sum_{t = 0}^{T-1} (1 - m_t) (m_t - b/\sqrt{d})^{s-1} - \delta \sum_{t = 0}^{T-1} \nabla^S M_t \cdot \theta^* - \delta^2 \sum_{t = 0}^{T-1} |m_t| \left|\nabla_\theta l_t\right|^2 \mathds{1}_{\{\left|\nabla_\theta l_t\right|^2 \leq \MM\}} \\
    &- \delta^2 \sum_{t = 0}^{T-1} |m_t| \left|\nabla_\theta l_t\right|^2 \mathds{1}_{\{\left|\nabla_\theta l_t\right|^2 > \MM\}}  - \delta^3 \sum_{t = 0}^{T-1} \xi_t, \nonumber
\end{align*}
where we use, for the sake of compactness, the shortcut notation $l_t = l(x_t, \theta_t)$. The strategy of the proof is the following: the first term is the drift term that makes the correlation grow, the second term is simply a martingale term that we deal with via standard martingale inequality, and the forth and fifth term are discretization error that we will bound loosely. The difficulty comes from the third term: the proof is based on the fact that we use a ``part'' of the drift term (say half) to control it. This is why we decide to rewrite finally our inequality as,    
\begin{align}
\label{eq:dynamics_lowersum_m}
    m_{T} \geq m_0 &+ \delta \frac{C}{2} \sum_{t = 0}^{T-1} (1 - m_t) (m_t - b/\sqrt{d})^{s-1} - \delta \sum_{t = 0}^{T-1} \nabla^S M_t \cdot \theta^* - \delta \sum_{t = 0}^{T-1} D_t \\
    &- \delta^2 \sum_{t = 0}^{T-1} |m_t| \left|\nabla_\theta l_t\right|^2 \mathds{1}_{\{\left|\nabla_\theta l_t\right|^2 > \MM\}}  - \delta^3 \sum_{t = 0}^{T-1} \xi_t, \nonumber
\end{align}
where we have defined $D_t := \frac{C}{2} (1 - m_t) (m_t - b/\sqrt{d})^{s-1} - \delta |m_t| \left|\nabla_\theta l_t\right|^2 \mathds{1}_{\{\left|\nabla_\theta l_t\right|^2 \leq \MM\}}$.
The following section show how to control these five terms in a quantitative way.

\subsection{Weak recovery}
\label{subsec:weak_recovery}

\paragraph{Good initialization.}\textit{During all this section, we condition on the event $\{m_0 \geq 5 b /\sqrt{d}\}$}.

Before stating these lemmas, let us introduce some new notations. As already introduce, we recall that we denote $S_\eta := \{\theta \in \S, \, m_\theta \geq \eta\}$, the spherical cap of level $\eta \in (0,1)$. Moreover for $\alpha \in (-1, 1) $, similarly to what is done in~\cite{arous2020online}, we define the following stopping times $\tau^+_\alpha := \inf\{ t \geq 0, \, m_{\theta_t} \geq \alpha  \}$ and $\tau^-_\alpha := \inf\{ t \geq 0, \, m_{\theta_t} \leq \alpha  \}$ reciprocally as the first time when $(\theta_t)_{t \geq 0}$ enters in $S_\alpha$ or leaves $S_\alpha$.

\subsubsection{Proof of Theorem~\ref{thm:weak_recovery}}

Thanks to Lemmas~\ref{lem:ODE_term}, \ref{lem:first_artingale_term}, \ref{lem:submartingale_term}, \ref{lem:first_discretization_term} and \ref{lem:second_discretization_term}, that serve bounding all the terms in the $m_T$ inequality, there exists a constant $K$ that depend solely on the model such that  we have the following lower bound: for all $\lambda > 0$, conditionally to the event on the events $\{T \leq \tau_{1/2}^+ \, \wedge \tau^-_{2b/\sqrt{d}}\ \}\,$,
\begin{align*}
    m_{T} &\geq m_0 + \frac{C}{2^{s+1}} \delta \sum_{t = 0}^{T-1} m_t^{s-1} - 4 \lambda,
\end{align*}
with probability larger that $\displaystyle 1 - \left(\frac{K T \delta^2}{\lambda^2} + \exp\left( -\frac{\lambda^2}{2 K^2 \delta^2 T + \lambda \delta (C + \delta \MM)}\right) +\frac{K T d^2 \delta^2}{\lambda \MM} +\frac{K T d \delta^3}{\lambda} \right)$. Now we choose $\lambda = b /\sqrt{d}$ and $\MM = d^{3/2}$ so that
\begin{align*}
    m_{T} &\geq \frac{b}{\sqrt{d}} + \frac{C}{2^{s+1}} \delta \sum_{t = 0}^{T-1} m_t^{s-1},
\end{align*}
with probability at least $1 - p_{\delta, M}(T)$, where we defined naturally
$$p_{\delta, M}(T) := \left(\frac{K T d \delta^2}{b^2} + \exp\left( -\frac{b^2}{2 K^2 d \delta^2 T + b \sqrt{d} \delta (C + \delta \MM)}\right) +\frac{K T d^{5/2} \delta^2}{b \MM} +\frac{K T d^{3/2} \delta^3}{b}  
 \right).$$ 
Let us upper bound the probability $p_{\delta, M}(T)$. Let us set $\varepsilon > 0$ a small constant. First, in the exponential term, the term $b \sqrt{d} \delta (C + \delta d^{3/2})$ is negligible in virtue of the fact that in any of the cases of Theorem~\ref{thm:weak_recovery}, we have $\delta \leq \varepsilon/d$. Moreover, for the sake of clarity, we gather all constant $K, C, b$ as one constant generic $\mathsf{K}$, as these depend only on the data distribution and the link function. Hence, for $d$ large enough,
\begin{align*}
p_{\delta, M}(T) &\leq \mathsf{K} \left(d T \delta^2 + \exp\left( -\frac{1}{d  T \delta^2 } \right) + d T \delta^2 + d^{3/2} T \delta^3 \right),
\end{align*}
and as $d^{3/2} T \delta^3 \lesssim d T \delta^2 $ for the range of $\delta$ we choose, we have $p_{\delta, M}(T)  \leq \mathsf{K} \left(d T \delta^2 + \exp\left( -\frac{1}{d  T \delta^2 } \right) \right)$, and considering that we will take in any case $d T \delta^2 \leq 1$, as we have the inequality $\exp\left( -\frac{1}{d  T \delta^2 }\right) \leq d T \delta^2$, so that finally
\begin{align*}
p_{\delta, M}(T)  \leq \mathsf{K} d T \delta^2
\end{align*}
%
We divide the proof into the three cases $s = 1,\, s=  2, \, s \geq 3$.

{\bfseries Case} $s = 1$, $\delta = \varepsilon/d$. In this case, we have that with probability $1 - p_{\delta, M}(T)$, 
\begin{align*}
    m_{T} &\geq \frac{b}{\sqrt{d}} + \frac{C \delta}{2^{s}} T. 
\end{align*}
The right and side is larger than $1/2$ as soon as $\delta T \geq 2^{s}/C$. From this we have that with probability at least $1 - p_{\delta, M}(T)$, the hitting time is upper bounded by
$$ \tau^+_{1/2} \leq \frac{2^s}{ C \delta}.$$
Now, taking $\delta = \varepsilon d^{-1}$, we can check that for $\varepsilon$  small enough, $d T \delta^2 \leq 2^s \varepsilon/C =\varepsilon \mathcal{O}(1) $ so that we have that with probability at least $\displaystyle 1 - \mathsf{K} \varepsilon$, we have 
$$ \tau^+_{1/2} \leq \frac{\mathsf{K}}{ \varepsilon } d.$$

{\bfseries Case} $s = 2$, $\delta = \varepsilon/(d\log d)$. Now by a discrete version of Grönwall inequality, recalled in Lemma~\ref{lem:gron_bihari}, we have with probability at least $1 - p_{\delta, M}(T)$,
\begin{align*}
    m_{T} - \frac{b}{\sqrt{d}} &\geq \frac{b}{\sqrt{d}} \left(  1 + \delta \frac{C}{2}\right)^{T} \geq \frac{b}{\sqrt{d}} e^{C \delta T}, 
\end{align*}
for $d$ large enough. And as the right hand side is larger than $1/2 + b / \sqrt{d}$ whenever,  $$\delta T \geq \frac{1}{C} \log (\sqrt{d}/4b),$$
for $d$ large enough compared to $b$. Then taking such a $T$, with probability at least $1 - p_{\delta, M}(T)$ , the hitting time is upper bounded by 
$$ \tau^+_{1/2} \leq \frac{2}{ C \delta} \log \left(d\right).$$
Now, taking $\delta = \varepsilon d^{-1}(\log d)^{-1}$, we can check that for $\varepsilon$  small enough, $d T \delta^2 \leq \frac{ \varepsilon }{C } =\varepsilon \mathcal{O}(1)$ so that we have that with probability at least $\displaystyle 1 - \mathsf{K} \varepsilon$, we have
$$ \tau^+_{1/2} \leq \frac{\mathsf{K}}{ \varepsilon } d \log(d)^2 .$$

{\bfseries Case} $s \geq 3$, $\delta = \varepsilon d^{-s/2}$. Now by the discrete version of Bihari-LaSalle inequality, recalled in Lemma~\ref{lem:gron_bihari}, we have with probability at least $1 - p_{\delta, M}(T)$, 
\begin{align*}
    m_{T} - \frac{b}{\sqrt{d}} &\geq \frac{b}{\sqrt{d}} \left(  1 - \delta \frac{C (s-2)}{2} \left(\frac{b}{\sqrt{d}}\right)^{ s - 2} T \right)^{-\frac{1}{s-2}}. 
\end{align*}
And as the right hand side is larger than $1/2 + b / \sqrt{d}$ whenever,  $$\delta T \geq \frac{d^{(s-2)/2}}{C (s - 2) b^{s-2}},$$
for $d$ large enough compare to $b$. Then taking such a $T$, with probability at least $1 - p_{\delta, M}(T)$, the hitting time is upper bounded by 
$$ \tau^+_{1/2} \leq \frac{1}{ C b^{s-2}} \frac{d^{\frac{s-2}{2}}}{\delta}.$$
Now, taking $\delta = \varepsilon d^{-s/2}$, we can check that for $\varepsilon$  small enough, $d T \delta^2 \leq \frac{ \varepsilon }{C (s - 2) b^{s-2} } =\varepsilon \mathcal{O}(1)$ so that we have that with probability at least $\displaystyle 1 - \mathsf{K} \varepsilon$, we have
$$ \tau^+_{1/2} \leq \frac{\mathsf{K}}{ \varepsilon} d^{s-1} .$$

\subsubsection{Technical intermediate result to lower bound each term of Eq.~\texorpdfstring{\eqref{eq:dynamics_lowersum_m}}{lol}}

\begin{lemma}[ODE term]
\label{lem:ODE_term}
   Conditioned to the event  $\{T \leq \tau_{1/2}^+ \, \wedge \tau^-_{2b/\sqrt{d}}\ \}\,$, we have the inequality 
    \begin{align*}
    \sum_{t = 0}^{T-1} (1 - m_t) (m_t - b/\sqrt{d})^{s-1} \geq \frac{1}{2^s} \sum_{t = 0}^{T-1} m_t^{s-1}.
    \end{align*}
\end{lemma}

\begin{proof}
    This simply results from the fact that for all $t \leq T - 1$, we have $\{t \leq \tau_{1/2}^+ \, \wedge \tau^-_{2b/\sqrt{d}}\ \}\, \subset \{T \leq \tau_{1/2}^+ \, \wedge \tau^-_{2b/\sqrt{d}}\ \}\,$, so that we can use the inequalities $1 - m \geq 1/2$ and $m - b/\sqrt{d} \geq m/2$. Summing these terms until $T - 1$ gives the proof of the lemma.
\end{proof}

\begin{lemma}[First martingale term]
\label{lem:first_artingale_term}
    For all $\lambda > 0$, we have that 
    \begin{align}
        \P \left( \sup_{t \leq T}\, \delta \left|\sum_{k = 0}^{t-1} \nabla^S M_k \cdot \theta^*\right|  \geq \lambda  \right) \leq \frac{K T \delta^2}{\lambda^2}, 
    \end{align}
    where $K > 0$, that depends solely on the model through $f, \nu$.
\end{lemma}

\begin{proof}
    This is a consequence of Doob's maximal inequality for (sub)martingale. Indeed, for $t \leq T $, let $H_{t-1} = \sum_{k = 0}^{t-1} \nabla^S M_k \cdot \theta^*$. We have that $H_t$ is a $\mathcal{F}_t$-adapted martingale and we have the following upper bounded on its variance:
\begin{align*}
     \E[H_{t-1}^2] &= \E\left[\left(\sum_{k = 0}^{t-1} \nabla^S M_k \cdot \theta^* \right)^2\right] \\
     &= \E\left[\sum_{k = 0}^{t-1} \left(\nabla^S M_k \cdot \theta^* \right)^2\right]  \\
     &\leq  t \sup_\theta \E_x\left[ \left(\nabla^S M_k \cdot \theta^* \right)^2\right] \\
     &\leq K t,
\end{align*}
where the last inequality comes from the Lemma~\ref{lem:tech_bounds}. Now, thanks to Doob's maximal inequality, we have for all $\lambda > 0$,
\begin{align*}
    \P \left( \sup_{t \leq T}\, \delta |H_{t-1}|  \geq \lambda  \right) \leq \frac{ \E[H_{T-1}^2] \delta^2}{\lambda^2} \leq \frac{K T \delta^2}{\lambda^2},
\end{align*}
and this concludes the proof of the lemma.
\end{proof}

\begin{lemma}[Submartingale term]
\label{lem:submartingale_term}
    For all $\lambda > 0$, if for all $t \leq T$, $m_t \in [2b/\sqrt{d}, 1/2]$, and $\delta$ is such that $\delta \leq \varepsilon / d$, with a small enough constant $\varepsilon>0$, we have that 
    \begin{align}
         \P\left( \delta\sum_{t=0}^{T-1} D_t \leq - \lambda \right)  \leq \exp \left( - \frac{\lambda^2}{ 2 K^2 \delta^2 T + \lambda \delta (C + \delta \MM)}\right)
    \end{align}
    where $K > 0$, that depends solely on the model through $f, \nu$.
\end{lemma}

\begin{proof}
    First, recall that we have defined $D_t = \frac{C}{2} (1 - m_t) (m_t - b/\sqrt{d})^{s-1} - \delta |m_t| \left|\nabla_\theta l_t\right|^2 \mathds{1}_{\{\left|\nabla_\theta l_t\right|^2 \leq \MM\}}$. Let us notice that if $m_t \in [2b/\sqrt{d}, 1/2]$, then $1-m_t \geq 1/2$ and $(m_t - b/\sqrt{d})^{s-1} \geq m_t^{s-1} / 2^{s-1}$. Hence, if $m_t$ lies in such an interval, 
    \begin{align*}
        D_t &\geq  \frac{C}{2^{s + 1}} m_t^{s-1} - \delta |m_t| \left|\nabla_\theta l_t\right|^2 \mathds{1}_{\{\left|\nabla_\theta l_t\right|^2 \leq \MM\}} \\
        &\geq  \frac{C}{2^{s + 1}} m_t^{s-1} \left(1 -  2^{s+1} \delta \frac{ \left|\nabla_\theta l_t\right|^2 \mathds{1}_{\{\left|\nabla_\theta l_t\right|^2 \leq \MM\}}}{C m_t^{s-2}} \right).
    \end{align*}
    Now, for $\delta$ such that $\E\left[ 1 -  2^{s+1} \delta \frac{ \left|\nabla_\theta l_t\right|^2 \mathds{1}_{\{\left|\nabla_\theta l_t\right|^2 \leq \MM\}}}{C m_t^{s-2}}\, | \, \mathcal{F}_{t-1}\right] \geq 0$, $\left(\sum_{k=1}^t D_k\right)_{t \geq 0}$ is a submartingale, which is true as soon as
    \begin{align*}
        \delta \leq \frac{C m_t^{s-2}}{2^{s+1}\sup_\theta \E\left[ \left|\nabla_\theta l_t\right|^2 \mathds{1}_{\{\left|\nabla_\theta l_t\right|^2 \leq \MM\}}\,  |\, \mathcal{F}_{t-1}  \right]  }, 
    \end{align*}
    which is itself true if
    \begin{align*}
        \delta \leq \frac{C}{4^{s}\sup_\theta \E\left[ \left|\nabla_\theta l_t\right|^2\right] }, 
    \end{align*}
    which is implied by the condition required in the lemma given the upper bound on~$\E[ \left|\nabla_\theta l_t\right|^2]$ provided in Lemma~\ref{lem:tech_bounds}.
    In order to apply Freedman tail inequality for this submartingale, let us provide upper bound on the increments as well as their variance. Indeed, we have, for all $t \geq 0$,  
    \begin{align*}
       |D_t| &\leq \frac{C |1 - m_t| |m_t - b \sqrt{d}|^{s-1}}{2} + \delta |m_t| \left|\nabla_\theta l_t\right|^2 \mathds{1}_{\{\left|\nabla_\theta l_t\right|^2 \leq \MM\}}  \\
       &\leq  \frac{C + \delta \MM}{2}~,
    \end{align*}
       and in virtue of the inequality $(a+b)^2 \leq 2(a^2 + b^2)$, we have
    \begin{align*}
    \E\left[D_t^2\,|\, \mathcal{F}_{t-1}\right] &\leq 2 \left( \frac{C^2 |1 - m_t|^2 |m_t - b \sqrt{d}|^{2(s-1)}}{4} + \delta^2 |m_t|^2 \E \left[\left|\nabla_\theta l_t\right|^4 \mathds{1}_{\{\left|\nabla_\theta l_t\right|^2 \leq \MM\}} \right]  \right) \\
     &\leq \frac{C^2 + \delta^2 \E\left[ \left|\nabla l\right|^4\right]}{2}  \\
     & \leq \frac{C^2 + K \delta^2 d^2}{2} \\
     & \leq K^2. 
    \end{align*}
    Hence, by the Freedman tail inequality recalled in Theorem~\ref{thm:freedman}, for all $\lambda > 0$, 
     \begin{align*}
    \P\left( \delta\sum_{t=0}^{T-1} D_t \leq - \lambda \right)  \leq \exp \left( - \frac{\lambda^2}{ 2 K^2 \delta^2 T + \lambda \delta (C + \delta \MM)}\right)~,
    \end{align*}
    which concludes the proof of the Lemma.
\end{proof}

\begin{lemma}[First discretization term]
\label{lem:first_discretization_term}
    We have that, almost surely 
    \begin{align}
         \P \left( \sup_{t \leq T}\, \delta^2 \sum_{t = 0}^{T-1} |m_t| \left|\nabla_\theta l_t\right|^2 \mathds{1}_{\{\left|\nabla_\theta l_t\right|^2 > \MM\}} \geq \lambda \right) \leq \frac{K T \delta^2 d^2}{\lambda \MM}, 
    \end{align}
    where $K > 0$ depends solely on the model through $f, \nu$.
\end{lemma}
\begin{proof}
This term is handled via a combination of Markov and Cauchy-Schwartz inequalities. First, notice that,
\begin{align*}
    \sup_{t \leq T}\, \delta^2 \sum_{t = 0}^{T-1} |m_t| \left|\nabla_\theta l_t\right|^2 \mathds{1}_{\{\left|\nabla_\theta l_t\right|^2 > \MM\}} \leq T \delta^2 \sup_{t \leq T} \left\{ |m_t|\left|\nabla_\theta l_t\right|^2 \mathds{1}_{\{\left|\nabla_\theta l_t\right|^2 > \MM\}} \right\}.
\end{align*}
Furthermore, for all $t \leq T$, all $\lambda > 0$, via Markov inequality, then Cauchy-Schwartz inequality,
\begin{align*}
\P\left(|m_t|\left|\nabla_\theta l_t\right|^2 \mathds{1}_{\{\left|\nabla_\theta l_t\right|^2 > \MM\}} \geq \lambda \right) &\leq \frac{\E\left[|m_t|\left|\nabla_\theta l_t\right|^2 \mathds{1}_{\{\left|\nabla_\theta l_t\right|^2 > \MM\}}\right]}{\lambda} \\
&  \leq \frac{\sqrt{\E\left[\left|\nabla_\theta l_t\right|^4\right]}\sqrt{ \P\left(\left|\nabla_\theta l_t\right|^2 > \MM\right)}}{\lambda} \\
&  \leq \frac{\sqrt{\E\left[\left|\nabla_\theta l_t\right|^4\right]}\sqrt{ \E\left[\left|\nabla_\theta l_t\right|^4\right] / \MM^2 }}{\lambda} \\
&  \leq \frac{\E\left[\left|\nabla_\theta l_t\right|^4\right]}{\lambda \MM} \\
&  \leq \frac{K d^2}{\lambda \MM}~,
\end{align*}
where the last inequality is due to Lemma~\ref{lem:tech_bounds}. Multiplying this bound by $T \delta^2$ ends the proof the lemma.
\end{proof}

\begin{lemma}[Second discretization term]
\label{lem:second_discretization_term}
    For all $\lambda > 0$, we have that 
    \begin{align}
        \P \left( \sup_{t \leq T}\, \delta^3 \sum_{k = 0}^{t-1} \xi_k  \geq \lambda  \right) \leq \frac{K T d \delta^3}{\lambda}, 
    \end{align}
    where $K > 0$ depends solely on the model through $f, \nu$.
\end{lemma}
\begin{proof}
   Recall that $\xi_k = \left|\nabla_\theta l(\theta_k, x_k)\right|^2 |\nabla^S l(\theta_k, x_k) \cdot \theta^* |$. The bound follows from an application of Markov's inequality. Indeed, since all the terms of the sum are positive, the supremum is attained in $t = T-1$, and we shall only consider this case. For $\lambda > 0$, 
    \begin{align*}
        \P \left( \delta^3 \sum_{t = 0}^{T-1} \xi_t  \geq \lambda  \right) &\leq \frac{\delta^3}{\lambda}\E\left[ \sum_{t = 0}^{T-1} \xi_t \right]  \\
        &\leq \frac{T \delta^3}{\lambda}\sup_\theta \left\{  
        \E_x[\left|\nabla_\theta l(\theta, x)\right|^2 |\nabla^S l(\theta, x) \cdot \theta^* |]\right\} \\ 
        &\leq \frac{T \delta^3}{\lambda}\sup_\theta \left\{ 
        \sqrt{\E_x\left[[\left|\nabla_\theta l(\theta, x)\right|^4\right]} \sqrt{ \E_x \left[ |\nabla^S l(\theta, x) \cdot \theta^* |^2 \right]} \right\} \\  
        &\leq \frac{T \delta^3}{\lambda} 
        \sqrt{\sup_\theta \E_x\left[[\left|\nabla_\theta l(\theta, x)\right|^4\right]} \sqrt{ \sup_\theta \E_x \left[ |\nabla^S l(\theta, x) \cdot \theta^* |^2 \right]}  \\  
        &\leq \frac{T \delta^3}{\lambda} 
        \sqrt{\sup_\theta \E_x\left[[\left|\nabla_\theta l(\theta, x)\right|^4\right]} \sqrt{ \sup_\theta \E_x \left[ |\nabla^S l(\theta, x) \cdot \theta^* |^2 \right]}  \\ 
        &\leq \frac{T \delta^3}{\lambda} 
        \sqrt{K d^2} \sqrt{ K }  \\ 
        &\leq \frac{K T d \delta^3}{\lambda},
    \end{align*}
    where the penultimate inequality comes from Lemma~\ref{lem:tech_bounds}. 
\end{proof}

\subsection{Strong recovery}
\label{subsec:strong_recovery}

The reasoning is almost identical to the one of the previous section, except from the fact that instead of tracking the growing movement on $(m_t)_{t \geq 0}$, we will track the decaying movement of $(1 - m_t)_{t \geq 0}$. 

\subsubsection{Upper bound on the residual}

As said in the main text, we place ourselves \textit{after} the weak recovery time. Thanks to the Markovian property of the SGD dynamics, we have the equality between all time $s > 0$ marginal laws of
\begin{align*}
\left(\theta_{\tau^+_{\nicefrac{1}{2}} + s}\  \bigg| \  \tau^+_{1/2},\, \theta_{\tau^+_{1/2}}\right) \overset{\text{Law}}{=} \left(\theta_{s} \  \bigg| \  \theta_{s} = \theta_{\tau^+_{1/2}} \right),  
\end{align*}
and hence the strong recovery question is equivalent to study the dynamics with initialization such that $m_\theta  = 1/2$. As demonstrated before we have that $\P(\tau_{1/2}^+ < \infty) \geq 1 - \mathsf{K} \varepsilon$ so that up to $\varepsilon$ terms, this conditioning does not hurt the probability of the later events. In fact this conditioning seems even artificial as it seems provable that $\tau_{1/2}^+$ is almost surely finite. Yet, we leave this more precise study for another time. 

\subsubsection{A (slightly) different decomposition}

Let us define for all $t \in \N $, the residual $u_t = 1 - m_{t+\tau^+_{1/2}} > 0$, and thanks to the lower bound given by Eq.~\eqref{eq:dynamics_lower_m}, we have
\begin{align*}
    u_{t+1} \leq u_t - C \delta u_t (m_t - b/\sqrt{d})^{s-1} + \delta\nabla^S M_t \cdot \theta^* + \delta^2 |m_t| |\nabla l(x_t, \theta_t)|^2  +  \delta^3 \xi_t,
\end{align*}
From there, the proof is similar to the weak recovery case, except that the extra-care we used for the term $\delta^2 |m_t| |\nabla l(x_t, \theta_t)|^2$ is not necessary. We use simply the decomposition of this term in a second martingale term $$N_t = |m_t| |\nabla l(x_t, \theta_t)|^2 - \E\left[|m_t| |\nabla l(x_t, \theta_t)|^2 | \mathcal{F}_{t-1}\right]$$ and the drift that we directly upper bound as $\E\left[|m_t| |\nabla l(x_t, \theta_t)|^2 | \mathcal{F}_{t-1}\right] \leq K d $. Now similarly to Lemma~\ref{lem:first_artingale_term}, we have the upper bound:
\begin{lemma}[New martingale term]
\label{lem:second_artingale_term}
    For all $\lambda > 0$, we have that 
    \begin{align}
        \P \left( \sup_{t \leq T}\, \delta^2 \left|\sum_{k = 0}^{t-1} N_k \right|  \geq \lambda  \right) \leq \frac{K d^2 T \delta^4}{\lambda^2}, 
    \end{align}
    where $K > 0$, that depends solely on the model through $f, \nu$.
\end{lemma}

\begin{proof}
 This is a consequence of Doob's maximal inequality for the martingale. Indeed, for $t \leq T $, let $H_{t-1} = \sum_{k = 0}^{t-1} N_k$. We have that $N_t$ is a $\mathcal{F}_t$-adapted martingale and we have the following upper bounded on its variance:
\begin{align*}
     \E[N_{t-1}^2] &= \E\left[\left(\sum_{k = 0}^{t-1} \nabla^S M_k \cdot \theta^* \right)^2\right] \\
     &= \E\left[\sum_{k = 0}^{t-1} N_k ^2\right]  \\
     &\leq  t \sup_\theta \E_x \left(N_k \right)^2 \\
     &\leq K d^2 t,
\end{align*}
where the last inequality comes from the Lemma~\ref{lem:tech_bounds}. Now, thanks to Doob's maximal inequality, we have for all $\lambda > 0$,
\begin{align*}
    \P \left( \sup_{t \leq T}\, \delta^2 |H_{t-1}|  \geq \lambda  \right) \leq \frac{ \E[H_{T-1}^2] \delta^4}{\lambda^2} \leq \frac{K d^2 T \delta^4}{\lambda^2},
\end{align*}
and this concludes the proof of the lemma.
\end{proof}
Now, everything is in order to prove the Theorem~\ref{thm:strong_recovery}.
\subsubsection{Proof of Theorem~\ref{thm:strong_recovery}}
Let us fix a small number $\varepsilon>0$. As previously, thanks to Lemmas~\ref{lem:first_artingale_term}, \ref{lem:second_discretization_term},  \ref{lem:second_artingale_term}, there exists $K>0$ that depends solely on the model such that we have the following upper bound: for all $\lambda$, and $ t \leq \tau_{1/3}^- \wedge \tau_{1-\varepsilon}^+ $ summing between times $0$ and $t$, 
\begin{align*}
    u_{t} &\leq u_0 - \frac{C \delta}{4^{s-1}} \sum_{k =0}^{t-1} u_k + K \delta^2 d + 3\lambda,
\end{align*}
with probability larger that $\displaystyle 1 - \left(\frac{K t \delta^2}{\lambda^2} + \frac{K d^2 t \delta^4}{\lambda^2} +\frac{K t d \delta^3}{\lambda}\right)$ and $d$ large enough. Let us choose $\lambda = 1/16$ and $\delta$ small enough so that $K\delta^2 d \leq \lambda$. Hence, realizing that $u_{0} \leq 1/2$, we have 
\begin{align*}
    u_{t} &\leq \frac34 - \frac{C \delta}{4^{s-1}} \sum_{k =0}^{t-1} u_k~,
\end{align*}
with probability at least $1 - \mathsf{K} t \delta^2 ( 1  + d^2 \delta^2 + d \delta) \gtrsim 1 - \mathsf{K} t \delta^2$, as we choose in any case $\delta = \varepsilon \mathcal{O}(1)$. Note that we used the same convention as in the weak recovery case that $\mathsf{K}$ denotes \textit{any} constant that simply depend on the model. We have by Grönwall inequality~(Lemma~\ref{lem:gron_bihari})
\begin{align*}
    u_{t} &\leq \frac34 \left(1 - \frac{C \delta}{4^{s-1}} \right)^{t} \leq \frac34 e^{- \frac{C \delta}{4^{s-1}}t }.
\end{align*}
Hence, as the right end side is smaller than $\varepsilon$ for the time $$t \delta \geq \frac{4^{s-1}}{C} \log(1/\varepsilon),$$
we choose such a $t$, so that with probability at least $1 - \mathsf{K} \delta \log(1/\varepsilon)$, the delayed hitting time $\overline{\tau}^+_{1-\varepsilon} := \inf\{ t \geq 0, \, u_{t} \leq \varepsilon  \}$ satisfies
$$\overline{\tau}^+_{1-\varepsilon} \leq \frac{4^{s-1}}{C \delta} \log(1/\varepsilon),$$ 
and taking $\delta = \varepsilon / d$ gives that with a probability at least $1 - \mathsf{K} \varepsilon \log(1/\varepsilon) / d$, we have 
$$ \overline{\tau}^+_{1-\varepsilon} \leq \frac{4^{s-1}}{C \varepsilon} d \log(1/\varepsilon). $$ 
Considering that $d$ is large and $\varepsilon$ is simply a constant we get that $1 - \mathsf{K} 
  \varepsilon \log(1/\varepsilon) / d \geq 1 - \mathsf{K}\varepsilon $  and and this concludes the proof of Theorem~\ref{thm:strong_recovery}.

\subsection{Some technical bounds}
\label{subsec:tech_SGD_bounds}

We end this section by providing (i) some necessary technical technical bound on the quantities appearing in the SGD controls (ii) some discrete versions of Grönwall-type lemmas.

\subsubsection{Technical bounds on models expectations}

\begin{lemma}[Technical bounds]
\label{lem:tech_bounds}
 We have that there exists a constant $K > 0$ solely depending on the function $\phi$ and the distribution $\nu$ such that:
 \begin{align}
    \ \sup_{\theta \in \mathcal{S}_{d-1}} \E_x\left[ \langle \nabla_\theta^{\mathcal{S}} M(x, \theta), \theta_*\rangle^2\right] & \leq K~, \quad \text{and } \quad  \sup_\theta \E_x \left[ |\nabla^S l(\theta, x) \cdot \theta^* |^2 \right]] \leq K  \\
    \sup_{\theta \in \mathcal{S}_{d-1}} \E_x[\left|\nabla_\theta l(\theta, x)\right|^2] &\leq K d, \\
    \sup_{\theta \in \mathcal{S}_{d-1}} \E_x[\left|\nabla_\theta l(\theta, x)\right|^4] &\leq K d^2.
 \end{align}
\end{lemma}

\begin{proof}

In all the following proof we consider any $\theta \in \S$. Notice that we have the following calculation that is common to all the bounds we cover
\begin{align*}
    \nabla l (\theta, x) = x \phi'(x \cdot \theta)  \phi(x \cdot \theta_*)
\end{align*}

We treat the three bounds separately. 
%

\textit{First terms.} We have that for all $x \in \R^d$,   
    \begin{align*}
        M(x, \theta) &= l(x, \theta) - \E_\nu[l(x, \theta)],
    \end{align*}
hence 
    \begin{align*}
        \nabla^{\mathcal{S}}_\theta M(x, \theta) &=  \nabla^{\mathcal{S}}_\theta l(x, \theta) - \E_\nu[ \nabla^{\mathcal{S}}_\theta l(x, \theta)] \\
        &=  \nabla_\theta l(x, \theta) - \E_\nu[ \nabla_\theta l(x, \theta)] - (\theta \cdot \nabla_\theta l(x, \theta)) \theta + \E_\nu[ (\theta \cdot \nabla_\theta l(x, \theta)) \theta]
        ,
    \end{align*}
    and finally, 
    \begin{align*}
        \nabla^{\mathcal{S}}_\theta M(x, \theta) \cdot \theta_* &=     \nabla_\theta l(x, \theta)\cdot \theta_* - \E_\nu[ \nabla_\theta l(x, \theta)\cdot \theta_*] - (\theta \cdot \nabla_\theta l(x, \theta)) m + \E_\nu[ (\theta \cdot \nabla_\theta l(x, \theta)) m].
    \end{align*}
    hence thanks to applying the inequality $(a + b)^2 \leq 2 a^2 + 2 b^2$, this amounts to bound first
      \begin{align*}
        \E_x \left(\nabla_\theta l(x, \theta)\cdot \theta_* \right)^2 &=   \E_x \left[ \left( x \cdot \theta_*\right)^2 \phi'^2(x \cdot \theta)  \phi^2(x \cdot \theta_*) \right] \leq K,
    \end{align*}
    and second
      \begin{align*}
        \E_x \left((\nabla_\theta l(x, \theta)\cdot \theta) m \right)^2 &\leq   \E_x \left[ \left( x \cdot \theta\right)^2 \phi'^2(x \cdot \theta)  \phi^2(x \cdot \theta_*) \right] \leq K.
    \end{align*}

    \textit{Second term.} We have 
        \begin{align*}
        \E_x \left|\nabla_\theta l(x, \theta)\right|^2 &=  \E_x \left[|x|^2 \phi'^2(x \cdot \theta)  \phi^2(x \cdot \theta_*) \right] \leq Kd~.
    \end{align*}
        \textit{Third term.} We have similarly
        \begin{align*}
        \E_x \left|\nabla_\theta l(x, \theta)\right|^4 &=  \E_x \left[|x|^4 \phi'^2(x \cdot \theta)  \phi^2(x \cdot \theta_*) \right] \leq Kd^2~.
    \end{align*}

\end{proof}

\subsubsection{Standard tail probabilities for submartingales}

We recall here a theorem on submartingales from Freedman. This is an adaptation from Theorem 4.1 stated in~\cite{freedman1975tail}.
\begin{theorem}[Submartinagle tail bound]
\label{thm:freedman}
    Suppose that $(X_t)_{t \in \N}$ is random sequence adapted to a filtration $(\mathcal{F}_t)_{t \in \N}$. For $T \geq 1$, suppose there exist $a, b > 0$ such that $\E[X_t\, |\, \mathcal{F}_{t-1}] \geq 0$, the almost sure upper-bound $\sup_{t\leq T} |X_t| \leq a $ as well as $\sup_{t\leq T} \E[X_t^2\, |\, \mathcal{F}_{t-1}] \leq b$, then for all $\lambda > 0$, 
    \begin{align}
    \P\left( \sum_{k=1}^T X_k \leq - \lambda \right)  \leq \exp \left( - \frac{\lambda^2}{2(T b + \lambda a)}\right)
    \end{align}
\end{theorem}

\subsubsection{Discrete Grönwall and Bihari-Lasalle bounds}

We now turn to stating a classical comparison lemma for recursive inequalities.

\begin{lemma}[Grönwall and  Bihari-Lasalle]
\label{lem:gron_bihari}
We have the bounds for the recursive inequalities:

{\bfseries Case $s = 2$.} Suppose $(m_t)_{t \in \N}$ satisfies for $s \geq 3$, and positives numbers $a,b>0$, and $b< a/2 \wedge 1$,
\begin{align}
    m_t &\geq a + b \sum_{k = 0}^{t - 1} m_k, \qquad \text{ then, } \ \ m_t \geq a \left( 1 + b \right)^{t} \\
    m_t &\leq a - b \sum_{k = 0}^{t - 1} m_k, \qquad \text{ then, } \ \ m_t \leq a \left( 1 - b \right)^{t}.
\end{align}

{\bfseries Case $s \geq 3$.} Suppose $(m_t)_{t \in \N}$ satisfies for $s \geq 3$, and positives numbers $a,b>0$:
\begin{align}
    m_t \geq a + b \sum_{k = 0}^{t - 1} m_k^{s-1}, \qquad \text{ then, } \ \ m_t \geq a \left( 1 - (s-2)b a^{s-2} t\right)^{-\frac{1}{s-2}}.
\end{align}
\end{lemma}
\begin{proof}
The case $s = 2$ is known to be the discrete version of the Grönwall lemma and is treated in all standard textbooks, the case $s \geq 3$ referred to as the Bihari-Lasalle inequality is for example proven in Appendix C of \cite{arous2021online}. 
\end{proof}

\section{The \texorpdfstring{$\LPG$}{lol} property in the symmetric case: proofs of Section \ref{subsec:sym}}
\label{sec:symapp}

\subsection{Useful Facts about Gegenbauer Polynomials}

We recall known facts on Gegenbauer Polynomials.

\paragraph{Definitions.}
Recall that $P_{j,d}$ denotes the Gegenbauer polynomial of degree $j$ and dimension $d$, normalized so that $P_{j,d}(1) = 1$ for all $j,d$. 
We denote also $\bar{P}_{j,\lambda}$ the Gegenbauer polynomials normalized so that $\| \bar{P}_{j,\lambda} \|^2_{L^2(\mathbb{R}, u_{2\lambda+2})} = \pi 2^{1-2\lambda} \frac{\Gamma(j + 2\lambda)}{(j+\lambda)\Gamma^2(\lambda)\Gamma(j+1)}$. Throughout the proof, we will use either $d$, and from time to time the mute symbol $\lambda$ to denote the dimension variable of Gegenbauer polynomials. They satisfy the following recurrence: 
\begin{align}
\label{eq:recurrgegen}
    (j+1)\bar{P}_{j+1, \lambda}(t) & = 2(j+\lambda) t \bar{P}_{j,\lambda}(t)-(j+2\lambda-1)\bar{P}_{j-1,\lambda}(t)~,
\end{align}
with first terms: $ \bar{P}_{0, \lambda}(t) = 1$ and $\bar{P}_{1, \lambda}(t)  = 2\lambda t$.

\paragraph{Rodrigues Formula for Gegenbauer Polynomials.} The Gegenbauer polynomials can be represented as repeated derivatives of a simple polynome. 
\begin{proposition}[{\cite[Proposition 4.19]{frye2012spherical}}]
We have the formula
   \begin{equation}
       P_{j,d}(t) = \frac{(-1)^j}{2^j (j + (d-3)/2)_j }(1-t^2)^{(3-d)/2} \left(\frac{d}{dt}\right)^j (1-t^2)^{j+ (d-3)/2}~, 
   \end{equation} 
   where $(x)_j=\prod_{k=0}^{j-1} (x-k)$ is the falling factorial. 
\end{proposition}

\paragraph{Hecke-Funk Formula. } Recall that we use the notation $\tau_d$ to denote the uniform distribution on the sphere and $u_d$ the distribution of, e.g., its first coordinate: $u_d \propto (1-t^2)^{(d-3)/2} \mathds{1}_{[-1,1]}$.
\begin{theorem}[{\cite[Theorem 4.24]{frye2012spherical}}]
    For $\theta, \theta' \in \S$, $f \in L^2_{u_d}(\mathbb{R})$ and $j \in \mathbb{N}$,
    \begin{align}
        \langle f_\theta, (P_{j,d})_{\theta'} \rangle_{\tau_{d}} &= \Omega_{d-2} P_{j,d}( \theta \cdot \theta' ) \langle f, P_{j,d} \rangle_{u_d} \nonumber \\
        &= \Omega_{d-2} P_{j,d}( \theta \cdot \theta' ) \int_{-1}^1 f(t) P_{j,d}(t) (1-t^2)^{(d-3)/2} dt ~.
    \end{align}
\end{theorem}


\begin{fact}\label{fact:deriv}[Derivative Representation] We have the following derivation property for all $j,d$:
\begin{align}
    P_{j,d}'  &= \frac{j (j+d-2)}{(d-1)} P_{j-1,d+2}~.
\end{align}
\end{fact}
\begin{proof}
 Recall the normalization relationships   $\lambda = \frac{d}{2}-1$, 
$\bar{P}_{j,\lambda}(1) = \frac{\Gamma(j+2\lambda)}{\Gamma(j+1)\Gamma(2\lambda)}$ , 
$P_{j,d} = \frac{\Gamma(j+1)\Gamma(d-2)}{\Gamma(j+d-2)}\bar{P}_{j, \frac{d}{2}-1}$, as well as the identity
$\bar{P}'_{j,\lambda} = 2 \lambda \bar{P}_{j-1,\lambda+1}$. Thus, 
\begin{align}
    P_{j,d}' &=  \frac{\Gamma(j+1)\Gamma(d-2)}{\Gamma(j+d-2)}\bar{P}'_{j, \frac{d}{2}-1} \nonumber \\
            &=  2\frac{\Gamma(j+1)\Gamma(d-2)}{\Gamma(j+d-2)} (\frac{d}{2}-1) \bar{P}_{j-1,\frac{d}{2}} \nonumber \\ 
            & = (d-2)  \frac{\Gamma(j+1)\Gamma(d-2)}{\Gamma(j+d-2)}  \frac{\Gamma(j-1+d)}{\Gamma(j)\Gamma(d)} P_{j-1,d+2} \nonumber \\
            &= \frac{j (j+d-2)}{(d-1)} P_{j-1,d+2}
\end{align}
\end{proof}

We have the following bound of the location of the largest root $z_{j,d}$ of $P_{j,d}$:  
\begin{fact}[Bound on the Largest Root,{\cite[Corollary 2.3]{Area2004ZerosOG}}]
\label{fact:largestroot}
\begin{equation}
    z_{j,d} \leq \sqrt{\frac{(j-1)(j+d-4)}{(j+d/2-3)(j+d/2-2)}} \cos(\pi/(j+1)) ~.
\end{equation}
\end{fact}

And we have the following bound on the Taylor expansion of the Gegenbauer polynomials:
\begin{fact}[Taylor Upper bound beyond largest root]
\label{fact:gegenwellbeh}
    $$P_{j,d}(t) \geq ( t - z_{j,d})^{j}~,~\text{for } t \geq z_{j,d}~,$$
\end{fact}
\begin{proof}
    Note that all families of orthogonal polynomials have exclusively real, simple roots.  Therefore, by Rolle's theorem, the $j-1$ critical points of  $P_{j,d}$ must be interlaced with the $j$ zeroes.  So all zeroes of $P_{j,d}'$ are upper bounded by $z_{j,d}$.  Futhermore, by Fact~\ref{fact:deriv}, $P_{j,d}'$ is itself an orthogonal polynomial.  So applying this argument recursively, we see the zeros of $P_{j,d}^{(k)}$ for $k \leq j$ are all upper bounded by $z_{j,d}$.

    Note also by Fact~\ref{fact:deriv} that, because $P_{j,d}(1) = 1$ for any choice of $j$ and $d$, it follows that $P_{j,d}^{(k)}(1) > 0$.  This implies $P_{j,d}^{(k)}(z_{j,d}) > 0$, as in order to flip signs there would need to be a zero in the range $[z_{j,d}, 1]$ which we've confirmed above cannot exist.

    Now, consider a Taylor expansion

    \begin{align}
        P_{j,d}(t) = \sum_{i=0}^j c_i (t - z_{j,d})^i
    \end{align}

    Observe that $P_{j,d}^{(k)}(z_{j,d}) = k! c_k$, and therefore by the above argument we have $c_k > 0$.  So it remains to show that $c_j \geq 1$.

    Consider applying Fact~\ref{fact:deriv} repeatedly, then we have:

    \begin{align}
        P_{j,d}^{(j)}(1) &= \frac{j! \prod_{l=1}^j (j+d-3 + l)}{\prod_{l=1}^j (d-3 + 2l)} \\
        &= j! \prod_{l=1}^j \frac{j+d-3 + l}{d-3 + 2l}\\
        &\geq j!
    \end{align}

    And from the fact that $P_{j,d}^{(j)}(1) = j!c_j$, we conclude $c_j \geq 1$.
\end{proof}


\subsection{Proof of Proposition \ref{prop:basicdec}}
\label{secapp:basicdec}
\begin{proposition}[Loss representation, restated]
The $\beta_{j,d}$ defined in \eqref{eq:lossmain} have the integral representation
    \begin{equation}
        \beta_{j,d} = \langle \phi, \mathcal{K}_j \phi \rangle_{L^2(\mathbb{R},\eta)} ~,
    \end{equation}
    where $\mathcal{K}_j$ is a positive semi-definite integral operator of $L^2_\eta$ that depend solely on $\rho$ and $\phi$,     
    with kernel 
\begin{align}
\mathcal{K}_j(t,t') &= \frac{\Omega_{d-2}N(j,d)}{\Omega_{d-1}}  \int_0^\infty  P_j(r^{-1} t) P_j(r^{-1} t') \bar{u}_d(r^{-1} t) \bar{u}_d(r^{-1} t') \rho(dr)~,
\end{align}
where we defined the conditional density 
$$\bar{u}_d(r^{-1}t) = \frac{r^{-1} {u}_d(r^{-1} t)}{\int_0^\infty (r')^{-1} {u}_d((r')^{-1} t) \rho(dr')}~. $$
Moreover, we have 
\begin{equation}
   \mathbb{E}_\eta [\phi^2] = \frac{\Omega_{d-2}}{\Omega_{d-1}}\sum_j \beta_{j,d} = \frac{\Gamma((d-2)/2)}{\sqrt{\pi}\Gamma((d-1)/2)} \sum_j \beta_{j,d}~.
\end{equation}
\end{proposition}

\begin{proof}
The marginal conditioned on $\|x\|=r$ is precisely given by 
$\eta(x_1=t ~|~\|x\|=r) = r^{-1} u_d(r^{-1} t)$, so 
$$\eta(t) = \int_0^\infty r^{-1} u_d(r^{-1} t) \rho(dr)~.$$

We have 
\begin{align}
\alpha_{j,r} &= \| P_j\|^{-2} \int_{-1}^1 P_j(t) \phi^{(r)}(t) \tau_d(dt) =  \| P_j\|^{-2} \int_{-1}^1 P_j(t) \phi(r t) u_d(t) dt \nonumber\\
&= \frac{\Omega_{d-2}N(j,d)}{\Omega_{d-1}} r^{-1} \int_{-\infty}^\infty \phi(t) P_j(r^{-1} t) (1-r^{-2}t^2)^{(d-3)/2}_+ dt~, 
\end{align}
so
{\small
\begin{align}
    \beta_{j,d} &= \frac{\Omega_{d-2}N(j,d)}{\Omega_{d-1}}  \int_0^\infty r^{-2} \iint_{-\infty}^\infty \phi(t) P_j(r^{-1} t) (1-r^{-2}t^2)^{(d-3)/2}_+ \phi(t') P_j(r^{-1} t') (1-r^{-2}(t')^2)^{(d-3)/2}_+ dt dt' \rho(dr) \nonumber\\
    &= \langle \phi, \mathcal{K}_j \phi \rangle_{L^2(\mathbb{R}, \eta)}~,
\end{align} 
}
with the $L^2(\mathbb{R},\eta)$ positive semi-definite integral kernel operator 
{\small 
\begin{align}
\mathcal{K}_j(t,t') &= \frac{\Omega_{d-2}N(j,d)}{\Omega_{d-1}}  \eta(t)^{-1} \eta(t')^{-1} \int_0^\infty r^{-2} P_j(r^{-1} t) (1-r^{-2}t^2)^{(d-3)/2}_+ P_j(r^{-1} t') (1-r^{-2}(t')^2)^{(d-3)/2}_+ \rho(dr) \nonumber \\
&= \frac{\Omega_{d-2}N(j,d)}{\Omega_{d-1}}  \int_0^\infty  P_j(r^{-1} t) P_j(r^{-1} t') \bar{u}_d(r^{-1} t) \bar{u}_d(r^{-1} t') \rho(dr)~,
\end{align}}
where we defined the conditional density 
$$\bar{u}_d(r^{-1}t) = \frac{r^{-1} {u}_d(r^{-1} t)}{\int_0^\infty (r')^{-1} {u}_d((r')^{-1} t) \rho(dr')}~. $$
Finally, let us establish (\ref{eq:enercons1}). 
We have 
\begin{align}
\label{eq:enercons1}
    \mathbb{E}_\eta \phi^2 &= \mathbb{E}_\rho [\mathbb{E}_{x_1|\|x\|=r} \mathbb{E} \phi(x_1)^2 ] \nonumber \\
    & = \mathbb{E}_\rho [ \mathbb{E}_{u_d} (\phi^{(r)})^2] \nonumber \\
    & = \mathbb{E}_\rho \sum_j \alpha_{j,d,r}^2 \| P_j \|^2 \nonumber \\
    & = \mathbb{E}_\rho \sum_j \alpha_{j,d,r}^2 \frac{\Omega_{d-1}}{\Omega_{d-2} N(j,d)} \nonumber \\
    &= \frac{\Omega_{d-1}}{\Omega_{d-2}} \mathbb{E}_\rho \sum_j \bar{\alpha}_{j,r,d}^2 = \frac{\Omega_{d-1}}{\Omega_{d-2}} \sum_j \beta_{j,d}~.
\end{align}

    \end{proof}

\subsection{Proof of Proposition \ref{prop:infoexpo_sphere}}

\begin{proof}

If $\beta_{j,d}=0$ for $j < s$, then $\alpha_{j,r,d}=0$ for $j<s$ and $\rho$-ae $r$. We want to show that for any polynomial $Q$ of degree $j'<s$, we must have 
$\langle \phi, Q \rangle_\eta  = 0$. 

For each $r$, consider $Q^{(r)}(t) = Q(rt)$, which is also a polynomial of degree $j'<s$, and its decomposition as $Q^{(r)} = \sum_{j=0}^{j'} b_{j,j',r} P_{j,d}$, which only involves terms of degree $j'<s$ since Gegenbauer polynomials of degree up to $r$ span all polynomials of degree up to $r$.
We have 
\begin{align}
    \langle \phi, Q \rangle_\eta &= \mathbb{E}_\eta [ \phi(x) Q(x)] \nonumber \\
    &= \mathbb{E}_\rho \mathbb{E}_{x_1 | \|x\|=r} [\phi(x) Q(x)] \nonumber \\
    &= \mathbb{E}_\rho \mathbb{E}_{u_d} [\phi^{(r)}(x) Q^{(r)}(x)] \nonumber \\
    &= \mathbb{E}_\rho [\sum_{j \leq j'} b_{j,j',r} \alpha_{j,r,d} ] = 0~.
\end{align}
\end{proof}


\subsection{Proof of Proposition \ref{prop:sufficientcond}}
\sufficientcond*
\begin{proof}
    Assume first that there are $\bar{C}, \bar{\zeta}$ such that 
\begin{equation}
\label{eq:bi0}
P'_{s,d}(t) \geq \bar{C}( t - \bar{\zeta})^{s-1}~,~\text{ for } t \geq \bar{\zeta}~.        
\end{equation}
Now, let 
\begin{equation}
\label{eq:bi1}
B = \frac{1}{d-1} \sum_{j \geq s} \beta_{j,d} j(j+d-2)  \upsilon_{j-1,d+2} < 1~,
\end{equation}
and define
\begin{equation}
\label{eq:bi2}
    \zeta^* := \left( \frac{B}{\beta_{s,d}\bar{C}}\right)^{1/(s-1)} + \bar{\zeta}~.
\end{equation}
From (\ref{eq:bi0}), (\ref{eq:bi1}) and (\ref{eq:bi2}) we verify that $\ell'(m) = \sum_j \beta_{j,d} P'_{j,d}(m)$ satisfies, for $m \geq \zeta^*$, 
$$\ell'(m) \geq \beta_{s,d} \bar{C} \left((m-\bar{\zeta})^{s-1} - (m-\zeta^*)^{s-1} \right) \geq  \beta_{s,d} \bar{C} \left[\left(\frac{1-\bar{\zeta}}{1-\zeta^*} \right)^{s-1} -1 \right] (m-\zeta^*)^{s-1}~.$$

Finally, we have that for any $j,d$, the largest root $z_{j,d}$ satisfies $z_{j,d} \leq \sqrt{\frac{(j-1)(j+2d-2)}{(j+d-2)(j+d-1)}} \simeq j/\sqrt{d}$ and 
$$P_{j,d}(t) \geq \frac12 ( t - z_{j,d})^{s}~,~\text{for } t \geq z_{j,d}~,$$
which implies that 
\begin{equation}
    P'_{s,d}(t) \geq \frac{s(s+d-2)}{2(d-1)} ( t - z_{s-1,d+2})^{s-1} ~,~\text{ for } t \geq z_{s-1,d+2}~. 
\end{equation}
We thus have $\bar{C} = \frac{s(s+d-2)}{2(d-1)}$ with $\bar{\zeta} = z_{s-1,d+2}$. 

Finally, we verify that 
\begin{align}
\label{eq:bi3}
    \frac{d^{(s-1)/2}}{s(s+d-2)} \sum_{j > s} \beta_{j,d} j (j+d-2) \upsilon_{j-1,d+2} &\leq K 
\end{align}
ensures a local polynomial growth of order $s-1$ at scale $O(1/\sqrt{d})$. 
Indeed, plugging (\ref{eq:bi3}) into (\ref{eq:bi1}), together with $\beta_{s,d} \geq C$ yields 
\begin{align}
    \left( \frac{B}{\beta_{s,d}\bar{C}}\right)^{1/(s-1)} &\leq (CK)^{1/(s-1)}d^{-1/2} ~, 
\end{align}
which shows that $\zeta^* = O(1/\sqrt{d})$. 
Finally, we observe that $\bar{C} \geq s=\Theta(1)$ if $s<d$. 

\end{proof}

\subsection{Proof of Theorem \ref{coro:symmcase}}
\label{sec:proofthmsym}

\begin{proof}



To prove the theorem, we will establish the sufficient conditions of Proposition \ref{prop:sufficientcond} under our mild assumptions. 
The key technical results we need are explicit bounds for $\upsilon_{j,d}$ and for the sum $\sum_j j^2 \beta_{j,d}$, established in the following two lemmas. Since the parameter $\lambda = d/2-1$ is more convenient to express many relationships in Gegenbauer polynomials, we will adopt it in this proof instead of $d$, without loss of generality. 

\begin{lemma}[Control of $\upsilon_{j,\lambda}$]
\label{prop:upsicontrol}
We have 
\begin{equation}
    \upsilon_{j,\lambda} \lesssim  \begin{cases}
    \left[1 - \left(\frac{\lambda}{j+\lambda} \right)^2 \right]^{j/2}
        & \text{ if } j=\Omega(1) ~,\\
        \lambda^{\frac{(\alpha-1)\lambda^\alpha}{2}} & \text{ if } j = \Theta(\lambda^\alpha), \text{ with } 0< \alpha < 1~, \\
        e^{-\frac12 \lambda^{2-\alpha}} & \text{ if } j=\Theta(\lambda^\alpha) \text{ with } 1\leq \alpha < 3/2 ~, \\
        e^{-\lambda} & \text{ if } j = \Omega(\lambda^{3/2}) ~.
    \end{cases}
\end{equation}
\end{lemma}

\begin{lemma}[Decomposition of derivative]
\label{lemma:betas_derivative}
If $\phi \in L^2(\mathbb{R}, \mu)$ is such that $\phi' \in L^4(\mathbb{R}, \eta)$ and $\mathbb{E}_\rho [r^4] < \infty$, then 
$\beta_j = \langle \phi, \mathcal{K}_j \phi \rangle$ satisfies 
\begin{equation}
    \sum_j j^2 \beta_{j,d} \leq \frac{\Omega_{d-2}}{\Omega_{d-1}} \mathbb{E}_\rho [r^4]^{1/2} \|\phi'\|_{L^4(\eta)}^2 = O(1/d)~.
\end{equation}
\end{lemma}
Let $s = \inf\{ j ; \beta_{j,d} \neq 0\}$.  
We need to verify that 
 there exists a constant $K>0$ such that  
 \begin{equation}
 \label{eq:flower}
 \sum_{j > s} \beta_{j,d} j (j+d-2)  \upsilon_{j-1,d+2} \leq K d^{(3-s)/2}~.    
 \end{equation}

We will control the LHS by splitting it into appropriate regions, determined by $J_i$, $i\in \{1,2,3\}$.  
Let $\alpha = \frac{4}{1+s}$ and $J_1 = \frac{\lambda^\alpha}{2}$. 
From Lemma \ref{prop:upsicontrol}, part (i) we have that 
$\upsilon_{j,\lambda} \leq C \left( \frac{j(j+2\lambda)}{(j+\lambda)^2}\right)^{j/2}$,  
and in particular $\upsilon_{j,\lambda} \leq C \lambda^{(\alpha-1)j/2}$ for $j \leq J_1$.
As a result, using Lemma \ref{lemma:betas_derivative}, 
\begin{align}
  \sum_{j = s+1}^{J_1} \beta_j j (j+\lambda)  \upsilon_{j-1,\lambda+1} & \leq \lambda^{(\alpha-1)(s+1)/2} \sum_{j=s+1}^{J_1} \beta_j j (j+\lambda) \nonumber\\
  &\leq  \lambda^{(\alpha-1)(s+1)/2} (C_1 \lambda^{-1} + \lambda \sum_{j=s+1} ^{J_1} \beta_j j )\nonumber\\
  &\leq  \lambda^{(\alpha-1)(s+1)/2} (C_1 \lambda^{-1} + \lambda \sum_{j=s+1} ^{J_1} \beta_j j^2 )\nonumber\\
  & \leq \lambda^{(\alpha-1)(s+1)/2} C_2 \nonumber\\
  &\leq C_2 \lambda^{(3-s)/2}~.
\end{align}

Let $J_2 = \lambda$. 
We have 
\begin{align}
    \sum_{j=J_1+1}^{J_2} \beta_j j (j+\lambda) \upsilon_{j-1,\lambda+1} &\leq \lambda^{\frac{(\alpha-1)\lambda^\alpha}{2}} C_3  \nonumber \\
    & \leq C_3 \lambda^{(3-s)/2}~.
\end{align}

Let $J_3 = \lambda^{3/2}$. 
We have 
\begin{align}
    \sum_{j=J_2+1}^{J_3} \beta_j j (j+\lambda) \upsilon_{j-1,\lambda+1} &\leq e^{-\frac12 \sqrt{\lambda}} C_4  \nonumber \\
    & \leq C_4 \lambda^{(3-s)/2}~.
\end{align}
Finally, the remainder satisfies 
\begin{align}
    \sum_{j>J_3} \beta_j j (j+\lambda) \upsilon_{j-1,\lambda+1} &\leq C_5 (e/2)^{-{\lambda}} \nonumber \\
    & \leq C_5 \lambda^{(3-s)/2}~,
\end{align}
which proves (\ref{eq:flower}).

    
\end{proof}

\begin{proof}[Proof of Lemma \ref{prop:upsicontrol}]
We prove this result by analysing different regimes for $j$ and $\lambda$. Concretely, we claim the following:
\begin{claim}
\label{claim:upsi_mainclaim}
    We have the following regimes:
    \begin{enumerate}
        \item For $j = \Omega(1)$, we have 
        \begin{equation}
        \label{eq:upsi_mainclaim1}
         \upsilon_{j,\lambda} \lesssim \left[1 - \left(\frac{\lambda}{j+\lambda} \right)^2 \right]^{j/2}~.
        \end{equation}
        \item For $j = \Theta(\lambda^\alpha)$, with $0< \alpha < 1$, we have 
        \begin{equation}
            \label{eq:upsi_mainclaim2}
             \upsilon_{j,\lambda} \lesssim \lambda^{\frac{(\alpha-1)\lambda^\alpha}{2}}~.
        \end{equation}
        \item For $j = \Theta(\lambda^\alpha)$, with $1 \leq \alpha < 2$, we have 
        \begin{equation}
            \label{eq:upsi_mainclaim2bis}
             \upsilon_{j,\lambda} \lesssim e^{-\frac12 \lambda^{2-\alpha}}~.
        \end{equation}
        \item For $j = \Omega(\lambda^\alpha)$, with $\alpha > 3/2$, we have 
        \begin{equation}
            \label{eq:upsi_mainclaim3}
            \upsilon_{j,\lambda} \lesssim e^{-\lambda}~.
        \end{equation}
    \end{enumerate}
\end{claim}
To prove the first three regimes of Claim \ref{claim:upsi_mainclaim}, we control $\upsilon_{j,\lambda}$ based on the distribution of the roots of $P_{j,\lambda}$. 
We recall that $(z_{k,j,\lambda})_{k \leq j}$ denotes the roots of $P_{j,\lambda}$ in increasing order, and $z_{j,\lambda} = z_{j,j,\lambda}$ its largest root. 
\begin{lemma}[Representation of $P_{j,\lambda}$ in terms of its roots, {\cite[Lemma 2.1]{de2008local}}]
\label{lem:basiccontrolgegen}
We have 
\begin{equation}
    P_{j,\lambda}(t) = \begin{cases}
         \prod_{k=j/2}^j \frac{t^2 - z_{k,j,\lambda}^2}{1 - z_{k,j,\lambda}^2} & \text{ if } j \text{ even}, \\
         t \prod_{k=(j+1)/2}^j \frac{t^2 - z_{k,j,\lambda}^2}{1 - z_{k,j,\lambda}^2} & \text{ if } j \text{ odd }~.
    \end{cases}
\end{equation}
\end{lemma}
From this representation, we deduce that $\upsilon_{j,\lambda}$ can be calculated explicitly. Indeed, as the local maxima of $|P_{j,\lambda}(t)|$ are increasing \cite{szego1939orthogonal}, {\cite[\href{https://dlmf.nist.gov/18.14}{Eq (18.14.15)}]{NIST:DLMF}}, we have the following equation:
    \begin{equation}
    \label{eq:upsi}
        \upsilon_{j,\lambda} = - P_{j,\lambda}(z_{j-1,\lambda+1}) = - \begin{cases}
         \prod_{k=j/2}^j \frac{z_{j-1,\lambda+1}^2 - z_{k,j,\lambda}^2}{1 - z_{k,j,\lambda}^2} & \text{ if } j \text{ even}, \\
         z_{j-1,\lambda+1} \prod_{k=(j+1)/2}^j \frac{z_{j-1,\lambda+1}^2 - z_{k,j,\lambda}^2}{1 - z_{k,j,\lambda}^2} & \text{ if } j \text{ odd }~.
         \end{cases}
    \end{equation}
Let us focus first on the case $j$ even, for simplicity. 
We can rewrite (\ref{eq:upsi}) more conveniently as
$$\upsilon_{j,\lambda} = \frac{z_{j,\lambda}^2 - z_{j-1,\lambda+1}^2}{1-z_{j,\lambda}^2} \prod_{k=j/2}^{j-1} \frac{z_{j-1,\lambda+1}^2 - z_{k,j,\lambda}^2}{1 - z_{k,j,\lambda}^2}~.$$



For $\delta\in (0,z_{j-1,\lambda+1})$ let 
$$m(\delta, j, \lambda) := | \{ k \in \{j/2,j\} ; z_{k,j,\lambda} \geq \delta \} |$$ 
denote the number of zeros of $P_{j,\lambda}$ in the interval $(\delta, 1)$. 
Since the function $t \mapsto \frac{a^2 - t^2}{1-t^2}$ is decreasing in $t \in (0,a)$, 
we have
\begin{fact}
\label{fact:upsi_basicbound}
We have the upper bound:
    \begin{equation}
        \upsilon_{j,\lambda} \leq  \frac{z_{j,\lambda}^2 - z_{j-1,\lambda+1}^2}{1-z_{j,\lambda}^2} \inf_{\delta} \left(\frac{z_{j-1,\lambda+1}^2 - \delta^2}{1-\delta^2} \right)^{m(\delta, j, \lambda)}~.
    \end{equation}
\end{fact}
Letting $\delta = z_{j/2,j,\lambda}$ the smallest positive root of $P_{j,\lambda}$ we have
\begin{equation}
\label{eq:bhi}
    \upsilon_{j,\lambda} \leq \frac{z_{j,\lambda}^2 - z_{j-1,\lambda+1}^2}{1-z_{j,\lambda}^2} \left(\frac{z_{j-1,\lambda+1}^2 - z_{j/2,j,\lambda}^2}{1-z_{j/2,j,\lambda}^2} \right)^{j/2}~.
\end{equation}
We can thus obtain an explicit control on $\upsilon_{j,\lambda}$ from bounds on the zeros of the Gegenbauer polynomials. We complement the upper bound on the largest root (Fact \ref{fact:largestroot}) with lower bounds for all positive roots, as well as a sharp lower bound for its largest root \cite{DIMITROV20101793}:
\begin{theorem}[Upper and Lower bounds for Gegenbauer roots, {\cite[Theorem 2]{DIMITROV20101793}}]
\label{thm:gegenroots_2}
    Let 
    $$b_{j,\lambda} = j^3 + 2(\lambda - 1)j^2 - (3\lambda -5)j + 4(\lambda - 1)~,$$ $$a_{j,\lambda} = 2(j + \lambda-1)(j^2 + j(\lambda-1) + 4(\lambda+1))~\text{and}$$ 
    $$c_{j,\lambda} = j^2(j+2\lambda)^2 + (2\lambda + 1)(j^2 + 2(\lambda+3)j + 8(\lambda - 1))~.$$ Then for every $k,j,\lambda$ we have 
    \begin{equation}
        \frac{b_{j,\lambda} - (j-2)\sqrt{c_{j,\lambda}}}{a_{j,\lambda}} \leq z_{k,j,\lambda}^2 \leq \frac{b_{j,\lambda} + (j-2)\sqrt{c_{j,\lambda}}}{a_{j,\lambda}}~.
    \end{equation}
\end{theorem}
\begin{theorem}[Lower bound for largest root, {\cite[Section 2.3]{DRIVER20121200}}]
\label{thm:gegenroots_3}
\begin{equation}
    z_{j,\lambda}^2 > 1 - \frac{(2\lambda+1)(2\lambda+3)}{(j-1)(j+2\lambda+1)+(2\lambda+1)(2\lambda+3)}:= 1 - \frac{g_{j,\lambda}}{h_{j,\lambda}}~.
\end{equation}    
\end{theorem}
Rewriting Fact \ref{fact:largestroot} as 
$z_{j,\lambda}^2 \leq \frac{e_{j,\lambda}}{f_{j,\lambda}}~,$
with 
$$e_{j,\lambda} = (j-1)(j+2\lambda-2)~,~f_{j,\lambda} = (j+\lambda-2)(j+\lambda-1)~,$$
and using again the monotonocity of $t \mapsto \frac{t-p}{1-t}$ 
we can bound the first term in the RHS of (\ref{eq:bhi}) as  
\begin{equation}
\label{eq:bhi2}
\frac{z_{j,\lambda}^2 - z_{j-1,\lambda+1}^2}{1-z_{j,\lambda}^2} \leq \frac{e_{j,\lambda}/f_{j,\lambda} + g_{j-1,\lambda+1}/h_{j-1,\lambda+1} -1}{1-e_{j,\lambda}/f_{j,\lambda}}~.    
\end{equation}
For $j, \lambda = \omega(1)$, 
we have 
\begin{align*}
a_{j,\lambda} &\simeq 2j (j+\lambda)^2,& ~b_{j,\lambda} &\simeq j^2(j+2\lambda)~,~\sqrt{c_{j,\lambda}}\simeq j(j+2\lambda)~, \\
e_{j,\lambda} &\simeq j(j+2\lambda),&~f_{j,\lambda}&\simeq (j+\lambda)^2~,\\
g_{j,\lambda} &\simeq 4\lambda^2 ,&~ h_{j,\lambda} &\simeq j(j+2\lambda) + 4\lambda^2~,    
\end{align*}
and thus
\begin{equation}
\label{eq:bhi3}
\frac{z_{j,\lambda}^2 - z_{j-1,\lambda+1}^2}{1-z_{j,\lambda}^2} \lesssim \frac{3 j (j+2\lambda)}{j (j+2\lambda) + 4\lambda^2} \leq 3~.     
\end{equation}
Therefore, 
\begin{align}
\label{eq:upsi_smalljregime}
    \upsilon_{j,\lambda} & \leq 3 \left(\frac{\frac{e_{j-1,\lambda+1}}{f_{j-1,\lambda+1}} - \frac{b_{j,\lambda} - (j-2)\sqrt{c_{j,\lambda}}}{a_{j,\lambda}}}{1-\frac{b_{j,\lambda} - (j-2)\sqrt{c_{j,\lambda}}}{a_{j,\lambda}}} \right)^{j/2}~ \nonumber \\
    & \leq 3\left(\frac{a_{j,\lambda} e_{j-1,\lambda+1} - f_{j-1,\lambda+1}(b_{j,\lambda} - (j-2)\sqrt{c_{j,\lambda}})}{f_{j-1,\lambda+1}(a_{j,\lambda} - b_{j,\lambda} + (j-2)\sqrt{c_{j,\lambda}})}\right)^{j/2}~ \nonumber \\
    & = 3\left(\frac{2j(j+\lambda)^2 j(j+2\lambda) - (j+\lambda)^2 (j^2(j+2\lambda) - j^2(j+2\lambda)) }{(j+\lambda)^2(2j(j+\lambda)^2 - j^2(j+2\lambda) + j^2(j+2\lambda))}  \cdot (1 + o_{j,\lambda}(1)) \right)^{j/2} \nonumber \\
    & \lesssim \left(\frac{j(j+2\lambda)}{(j+\lambda)^2} \right)^{j/2} \nonumber \\
    &= \left[1 - \left(\frac{\lambda}{j+\lambda} \right)^2 \right]^{j/2}~.
\end{align}
As a direct consequence of (\ref{eq:upsi_smalljregime}), we immediately obtain Eqs (\ref{eq:upsi_mainclaim1}), (\ref{eq:upsi_mainclaim2}) and (\ref{eq:upsi_mainclaim2bis}). The case where $j$ is odd is treated analogously. 

Let us now study the regime $j = \omega(\lambda^{3/2})$. 
Given $z \in \mathbb{C}$ with $|z|<1$, Gegenbauer polynomials admit the following generating function {\cite[Section 3.32]{watson1922treatise}}:
\begin{equation}
\frac{1}{(1 - 2z \cos \theta + z^2)^\lambda } = \sum_{j \geq 0} \bar{P}_{j, \lambda}(\cos \theta) z^j ~.
\end{equation}
From this generating function, the Cauchy integral formula leads to 
 the following integral representation: 
\begin{fact}[{\cite[Eq (1.2)]{ursell2007integrals}}]
For any $0 < \rho < 1$, we have
\begin{equation}
\bar{P}_{j, \lambda}(\cos \theta) = \frac{1}{2\pi i} \oint_{|z|=\rho}  \frac{dz}{(1-2z \cos \theta + z^2)^{\lambda} z^{j+1}}~.    
\end{equation}
\end{fact}
Assume $j = \Theta(\lambda^\alpha)$, with $\alpha > 3/2$. 
We are interested in the above representation for  $\bar{\theta} = \arccos(z_{j-1,\lambda+1})$. 
From Theorem \ref{thm:gegenroots_3}, we have 
$z_{j-1,\lambda+1}^2 \geq 1 - d_{j,\lambda}/(2c_{j,\lambda})$, and thus 
$$\bar{\theta}^2 \lesssim \frac{d_{j,\lambda}}{2c_{j,\lambda}} \simeq \frac{32\lambda^2 j^4 }{16 j^6} = \frac{2\lambda^2}{j^2}~,$$
so $\bar{\theta} = O(\lambda/j)$. 
Combining this upper bound with the lower bound obtained from Fact \ref{fact:largestroot} we have $\bar{\theta} = \Theta(\lambda/j)$. 

Using $1-\cos \theta \simeq  \theta^2/2 \simeq  \lambda^2/j^2$ 
and 
\begin{align}
|1 - 2 z \cos \theta + z^2| & = \left|(1-z)^2 + 2z(1-\cos \theta) \right| \nonumber \\
&\geq | 1- z|^2 - 2|z|(1-\cos \theta) \nonumber \\
& \geq 1 - \rho \left(2 + \Theta\left(\frac{\lambda^2}{j^2}\right)\right) + \rho^2~, 
\end{align}
we have 
\begin{align}
    | \bar{P}_{j, \lambda}(\cos \theta) | \leq \inf_{0<\rho < 1}|\rho|^{-(j+1)} \left(1 - \rho(2 + c\lambda^2/j^2) + \rho^2 \right)^{-\lambda}:= g(\rho)~.
\end{align}
Optimizing the RHS over $\rho$ we obtain 
$\rho^* = \frac{j - (\sqrt{2}-1)\lambda}{j + 2\lambda}$; substituting, we obtain 
\begin{align}
    g(\rho^*) \simeq e^{-(1+\sqrt{2}) \lambda} \left(\frac{j+2\lambda}{\lambda(1+\sqrt{2})} \right)^{2\lambda}~.
\end{align}
As a result, it follows that 
\begin{align}
    P_{j, \lambda}(\cos \theta) &= \bar{P}_{j, \lambda}(\cos \theta) \frac{j! (2\lambda-1)!}{(2\lambda+j-1)!} 
\end{align}
satisfies, for $\theta = \Theta(\lambda/j)$ and $j = \omega(\lambda^{3/2})$, 
\begin{align}
    \log |P_{j, \lambda}(\cos \theta) | \simeq & j \log j - j + 2\lambda \log(2\lambda) - 2\lambda - (j +2\lambda) \log(2\lambda + j) + 2\lambda+j \nonumber \\
    &~ -(1+\sqrt{2}) \lambda + 2\lambda \log (j + 2\lambda) - 2\lambda \log(\lambda (1+\sqrt{2})) \nonumber \\
    \simeq & -(1+\sqrt{2}) \lambda~,
\end{align}
where we have used Stirling's approximation. 
This proves Eq (\ref{eq:upsi_mainclaim3}) and completes the proof of  Lemma \ref{prop:upsicontrol}. 
\end{proof}
\begin{proof}[Proof of Lemma \ref{lemma:betas_derivative}]
We have, using Fact \ref{fact:deriv}, that
\begin{align}
\label{eq:bound0}
    \mathbb{E}_\rho [ \mathbb{E}_{u_d} {(\phi^{(r)}}')^2] 
    &= \frac{\Omega_{d-1}}{(d-1)^2\Omega_{d-2}} \sum_j \mathbb{E}_\rho \left[  \bar{\alpha}_{j,d,r}^2 \left(j (j+d-2)\right)^2  \right] \nonumber\\
    &\geq \frac{\Omega_{d-1}}{\Omega_{d-2}} \sum_j j^2 \mathbb{E}_\rho \left[  \bar{\alpha}_{j,d,r}^2 \right]~.
\end{align}
And we can upper bound via
\begin{align}
    \mathbb{E}_\rho [ \mathbb{E}_{u_d} {(\phi^{(r)}}')^2] & = \mathbb{E}_\rho [ r^2 \mathbb{E}_{u_d} ({(\phi')^{(r)}})^2] \nonumber \\
    &= \mathbb{E}_\rho \mathbb{E}_{x_1|\|x\|=r} [ r^2 (\phi(x_1)')^2 ] \nonumber \\ 
    & \leq \sqrt{\mathbb{E}_\rho [r^4] \mathbb{E}_\eta (\phi')^4} ~,
\end{align}
where this last line is finite by our assumptions on $\phi$ and $\rho$,

so from (\ref{eq:bound0}) we conclude that 
\begin{equation}
    \sum_j j^2 \beta_{j,d} \leq \frac{\Omega_{d-2}}{\Omega_{d-1}} \sqrt{\mathbb{E}_\rho [r^4] \mathbb{E}_\eta (\phi')^4}~.
\end{equation}
\end{proof}

\section{The \texorpdfstring{$\LPG$}{lol} property in the non-symmetric case: proofs of Section \ref{subsec:nonsym}}
\label{sec:nonsymapp}

\subsection{Proof of Proposition \ref{prop:gradwassers}}
\asslipreg*
\asssubgaussian*

\gradwassers*

\begin{proof}

Recall the notation $\phi_\theta(x) = \phi(\langle x, \theta\rangle)$. Let $v = \theta^* - m \theta$. 
From the definition, we have that 
\begin{align}
    \langle \nabla_\theta^{\mathbb{S}} L(\theta), \theta^* \rangle &= 2 \mathbb{E}_\nu\left[ \phi'_\theta(\phi_\theta - \phi_{\theta^*}) \left(x \cdot v\right)\right] \nonumber \\
    &:= \mathbb{E}_\nu [g_{\theta, \theta^*}]~.
\end{align}
Since $\mathbb{E}_\gamma [g_{\theta, \theta^*}]$ is precisely $\bar{\ell}'(m)(1-m^2)$, we need 
to establish that 
\begin{equation}
    \sup_\theta \left|\mathbb{E}_\nu g_{\theta, \theta^*} - \mathbb{E}_\gamma g_{\theta, \theta*} \right| \leq C \sqrt{1-m^2} \widetilde{W}_{1,2}(\nu, \gamma) \left(\log  \widetilde{W}_{1,2}(\nu, \gamma) \right)^2~.
\end{equation}

Fix $\theta$ and let $P_{\theta, \theta^*}$ be the orthogonal projection onto the subspace spanned by $\theta, \theta^*$.  For $R>0$ we consider $A_R = \{ x \in \mathbb{R}^d; \| P_{\theta, \theta^*} x \| \leq R \}$.  
\begin{align}
    \left|\mathbb{E}_\nu g_{\theta, \theta^*} - \mathbb{E}_\gamma g_{\theta, \theta*} \right| &= \left| \int g_{\theta, \theta^*}(x) (\nu(dx) - \gamma(dx) ) \right| \nonumber \\
    &\leq \underbrace{\left| \int_{x \in A_R} g_{\theta, \theta^*}(x) (\nu(dx) - \gamma(dx) ) \right|}_{T_a} + \underbrace{\left| \int_{x \notin A_R} g_{\theta, \theta^*}(x) (\nu(dx) - \gamma(dx) ) \right|}_{T_b}~.
\end{align}

Let us first bound $T_a$. Denote by $v = \theta^* - m \theta$, with $\|v\|^2 = 1-m^2$ Since $\phi$ and $\phi'$ are Lipschitz and $|\phi''| \leq O((1+t)^{-1})$ by Assumption \ref{ass:lipreg}, we have that 
\begin{align}
    \nabla_x g_{\theta, \theta^*}(x) &= \phi''_{\theta} (\phi_\theta - \phi_{\theta^*}) x^\top v \theta + \phi'_\theta ( \phi'_\theta \theta - \phi'_{\theta^*} \theta^*) x^\top v + \phi'_\theta (\phi_\theta - \phi_{\theta^*}) v 
\end{align}
satisfies 
\begin{align}
    \| \nabla_x g_{\theta, \theta^*}(x) \| &\leq  2\|v\| C \mathrm{Lip}(\phi) R +  4 \|v\|\mathrm{Lip}(\phi)^2 R \nonumber \\
    &\leq C \|v\| R ~,
\end{align}
and as a result we have that $g_{\theta, \theta^*}$ is $C \|v\| R$-Lipschitz when restricted to $A_R$, and thus
\begin{equation}
    T_a \leq C R \|v\| \widetilde{W}_{1,2}(\nu, \mu)~.
\end{equation}
Let us now control the tail $T_b$. Since $x^\top v$ is $\sqrt{2} M \|v\|$-subgaussian and $\phi$ is Lipschitz, we have that $z = |g_{\theta, \theta^*}(x)|$ is $\tilde{M} \|v\|$-subexponential where $\tilde{M}$ only depends on $M$ and $L$. 
It follows that 
\begin{align}
    T_b &\leq R (\mathbb{P}_\nu(z \geq R) + \mathbb{P}_\gamma( z \geq R) ) \nonumber \\
    &\leq R \exp\left( - \frac{\beta}{\|v\|} R \right)~,
\end{align}
where $\beta$ is a constant that depends only on $\tilde{M}$. 
As a result, we have 
\begin{align}
    \left|\mathbb{E}_\nu g_{\theta, \theta^*} - \mathbb{E}_\gamma g_{\theta, \theta*} \right| \leq \inf_{R>0} \left(C R \|v\| \widetilde{W}_{1,2}(\nu, \mu) + R \exp\left( -\frac{\beta}{\|v\|} R \right) \right)~.
\end{align}
Setting 
$$R =  {-\|v\| \beta^{-1}} \log((C \|v\| \widetilde{W}_{1,2}(\nu, \gamma))) $$
we obtain
\begin{align}
 \label{eq:blink}
 \left|\mathbb{E}_\nu g_{\theta, \theta^*} - \mathbb{E}_\gamma g_{\theta, \theta*} \right| &\leq  2 C (1-m^2)  \beta^{-1} \left|\log(C \|v\| \widetilde{W}_{1,2}(\nu, \gamma))\right|  \widetilde{W}_{1,2}(\nu, \gamma) \nonumber \\
& \leq (1-m^2) O\left( \widetilde{W}_{1,2}(\nu, \gamma) \log ( \widetilde{W}_{1,2}(\nu, \gamma)^{-1} ) \right)~,
\end{align}
as claimed.

\end{proof}

\subsection{Proof of Proposition \ref{prop:grad_expo2}}
We leverage Proposition \ref{prop:gradwassers} and the fact that if 
$\phi$ has information exponent $s=2$, then $\bar{\ell}'(m) \simeq m$ for small $m$. 

We need to show that for $b = \Theta(\log d)$ we have 
\begin{equation}
\label{eq:lpgcond0}
\langle \nabla_\theta L(\theta), \theta^*\rangle \geq C \left( m - \frac{b}{\sqrt{d}}\right)~,~\text{for }  \frac{b}{\sqrt{d}} \leq m \leq \frac12~,   
\end{equation}
as well as
\begin{equation}
\label{eq:contractcond0}
    \langle \nabla_\theta L(\theta), \theta^*\rangle \geq C' (1-m^2) 
\end{equation}
for $m\geq \frac12$. 

From (\ref{eq:blink}) and $\widetilde{W}_{1,2}(\nu, \gamma) \leq C /\sqrt{d}$, 
we obtain 
\begin{align}
    \langle \nabla_\theta L(\theta), \theta^*\rangle &= \bar{\ell}'(m)(1-m^2) + \langle \nabla_\theta L(\theta), \theta^*\rangle - \bar{\ell}'(m)(1-m^2)   \nonumber \\
    &\geq 2 \alpha_2^2 m (1-m^2) - (1-m^2)\tilde{C} \widetilde{W}_{1,2}(\nu, \gamma) \log(\widetilde{W}_{1,2}(\nu, \gamma)^{-1}) \nonumber \\
    & \geq  \left(\alpha_2^2 m  - \tilde{C} \frac{C}{\sqrt{d}} \log(\sqrt{d}/C)\right)(1-m^2) \nonumber \\
    & \geq  \alpha_2^2 \left( m  - \tilde{C} \frac{C}{\alpha_2^2 \sqrt{d}} \log(\sqrt{d}/C)\right)(1-m^2) \nonumber \\
    & \geq \alpha_2^2 \left( m - \frac{\log d^{C'}}{\sqrt{d}} \right)(1-m^2)~,
\end{align}
which proves (\ref{eq:lpgcond0}) and (\ref{eq:contractcond0}).

\subsection{Proof of Proposition \ref{prop:stein}}

\assthird*
\propstein*

\begin{proof}

 Recall the notation $\phi_\theta(x) = \phi( \langle x, \theta\rangle)$, and,  using $v = \theta^* - m \theta$, 
 \begin{align}
 h_{\theta, \theta^*}(x) &:= \phi_\theta^2 - 2 \phi_\theta \phi_{\theta^*}~,\\
 g_{\theta, \theta^*}(x) &:= 2 \phi'_\theta ( \phi_\theta - \phi_{\theta^*}) ( x \cdot v )~, 
 \end{align}
so that
 \begin{align}
 \Delta_L(\theta) &= \mathbb{E}_\nu[ h_{\theta, \theta^*}(x)] - \mathbb{E}_\gamma[h_{\theta, \theta^*}(x)]~, \\
 \Delta_{\nabla L}(\theta) &= \mathbb{E}_\nu[ g_{\theta, \theta^*}(x)] - \mathbb{E}_\gamma[g_{\theta, \theta^*}(x)]~.    
 \end{align}

The result is obtained via the following Stein coupling method for product measures:
\begin{theorem}[Stein Coupling, {{\cite[Theorem 3.1]{rollin2013stein}}}]
\label{thm:rollin}
Let $X$ be a $d$-dimensional random vector of independent coordinates, such that $\mathbb{E} X = 0$, $\mathbb{E} [XX^\top] = I_d$ and $\mathbb{E}|X_i|^3 = \tau_i^3 < \infty$. If $Z$ is a standard Gaussian random vector, and $h: \mathbb{R}^d \to \mathbb{R}$ is three-times differentiable, then 
\begin{equation}
    | \mathbb{E} h(X) - \mathbb{E} h(Z) | \leq \frac{5}{6} \sum_{i=1}^d \tau_i^3 \| \partial_{x_i}^3 h\|_\infty~.
\end{equation}
\end{theorem}

We verify that, thanks to the decay assumptions in Assumption \ref{ass:extralip}, we have
\begin{align}
    \partial_{x_i}^3 g_{\theta, \theta^*}(x) &= \lambda_1(x) \theta_i^3 + \lambda_2(x) \theta_i^2 \theta^*_i + \lambda_3(x) \theta_i (\theta^*_i)^2 + \lambda_4(x) (\theta^*_i)^3~, \\
    \partial_{x_i}^3 h_{\theta, \theta^*}(x) &= \lambda_5(x) \theta_i^3 + \lambda_6(x) \theta_i^2 \theta^*_i + \lambda_7(x) \theta_i (\theta^*_i)^2 + \lambda_8(x) (\theta^*_i)^3~,
 \end{align}
where 
\begin{equation}
    \sup_{k \in \{1,2,3,4\}} |\lambda_k(x)| \leq C \|v\|~,~ \sup_{k \in \{5,6,7,8\}} |\lambda_k(x)| \leq \tilde{C}~.
\end{equation}
Observing by Cauchy-Schwartz that 
\begin{align*}
\max\left\{\sum_i |\theta_i|^2 |\theta_i^*| , \sum_i |\theta_i|^3 \right\} &\leq \|\theta\|_4^2 ~,\\
\max\left\{\sum_i |\theta^*_i|^2 |\theta_i| , \sum_i |\theta^*_i|^3 \right\} &\leq \|\theta^*\|_4^2~, 
\end{align*}
we obtain from Theorem \ref{thm:rollin} that 
$$\Delta_L(\theta) \leq C \chi(\theta, \theta^*)~,~\Delta_{\nabla L}(\theta) \leq C' \|v\| \chi(\theta, \theta^*)~,$$
as claimed. 
\end{proof}

\end{document}